\documentclass[twoside]{article}
\usepackage[preprint]{aistats2026}
\usepackage{natbib} 
\usepackage{xurl}

\usepackage[utf8]{inputenc} % allow utf-8 input
\usepackage[T1]{fontenc}    % use 8-bit T1 fonts
\usepackage{hyperref}       % hyperlinks
\usepackage{url}            % simple URL typesetting
\usepackage{booktabs}       % professional-quality tables
\usepackage{amsfonts}       % blackboard math symbols
\usepackage{nicefrac}       % compact symbols for 1/2, etc.
\usepackage{microtype}      % microtypography
\usepackage[utf8]{inputenc} % allow utf-8 input
\usepackage[T1]{fontenc}    % use 8-bit T1 fonts
\usepackage{hyperref}       % hyperlinks
\usepackage{url}            % simple URL typesetting
\usepackage{amsmath}
\usepackage{booktabs}       % professional-quality tables
\usepackage{amsfonts}       % blackboard math symbols
\usepackage{nicefrac}       % compact symbols for 1/2, etc.
\usepackage{microtype}      % microtypography
\usepackage{graphicx}
\usepackage{booktabs}
\usepackage{multirow}
\usepackage{amsthm}
\usepackage{siunitx} % For better number formatting
\usepackage{arydshln} % For dashed lines
\usepackage{subfigure}
\usepackage[table]{xcolor}
\definecolor{lightblue}{RGB}{220,230,241}  % light blue

\newtheorem{proposition}{Proposition}
\begin{document}



\twocolumn[

% \aistatstitle{Rethinking cross entropy for continual fine-tuning: policy gradient with entropy annealing}
\aistatstitle{
Policy Gradient with Adaptive Entropy Annealing for Continual Fine-Tuning}

% \aistatsauthor{ Author 1 \And Author 2 \And  Author 3 }

% \aistatsaddress{ Institution 1 \And  Institution 2 \And Institution 3 } ]
\aistatsauthor{Yaqian Zhang$^{1}$ \And Bernhard Pfahringer$^{1}$ \And Eibe Frank$^{1}$ \And Albert Bifet$^{1,2}$}
\aistatsaddress{1. AI Institute, University of Waikato \And 2. LTCI, Télécom Paris}
]
\begin{abstract}
{Despite their success, large pretrained vision models remain vulnerable to catastrophic forgetting when adapted to new tasks in class-incremental settings. Parameter-efficient fine-tuning (PEFT) alleviates this by restricting trainable parameters, yet most approaches still rely on cross-entropy (CE) loss—a surrogate for the 0–1 loss—to learn from new data. We revisit this choice and revive the true objective (0-1 loss) through a reinforcement learning perspective. By formulating classification as a one-step Markov Decision Process, we derive an Expected Policy Gradient (EPG) method that directly minimizes misclassification error with a low-variance gradient estimation. Our analysis shows that CE can be interpreted as EPG with an additional sample-weighting mechanism: CE encourages exploration by emphasizing low-confidence samples, while EPG prioritizes high-confidence ones. Building on this insight, we propose adaptive entropy annealing (aEPG), a training strategy that transitions from exploratory (CE-like) to exploitative (EPG-like) learning. aEPG-based methods outperform CE-based methods across diverse benchmarks and with various PEFT modules. More broadly, we evaluate various entropy regularization methods and demonstrate that lower entropy of the output prediction distribution enhances adaptation in pretrained vision models. 
The source code is provided in the supplementary materials.}

\end{abstract}
\section{Introduction}

{ Modern vision models remain vulnerable to catastrophic forgetting when trained on non-stationary data. Traditional continual learning (CL) methods mitigate this challenge through techniques like memory replay~\citep{delange2021continual,zhang2022simple,van2020brain,tiwari2022gcr}, regularization~\citep{buzzega2020dark,rebuffi2017icarl,kirkpatrick2017overcoming}, or parameter isolation~\citep{xu2018reinforced,mallya2018packnet,hung2019compacting}. With the advent of large pretrained vision transformers, parameter-efficient fine-tuning (PEFT) has emerged as an effective way to substantially reduce forgetting. By restricting the number of trainable parameters, PEFT methods achieve state-of-the-art CL performance without relying on replay~\citep{wang2022learning,wang2022dualprompt,smith2023coda}.  

  Recent studies extend PEFT to CL using prompts~\citep{wang2022dualprompt,wang2022learning}, LoRA~\citep{liang2024inflora,liu2025lora}, and adapters~\citep{gao2024beyond,yu2024boosting,wang2025self,ermis2022memory}. On the stability side (i.e., retaining knowledge from past tasks), many techniques such as EMA-based fast–slow learning and orthogonal subspace constraints with LoRA have been proposed~\citep{gao2023unified,liang2024inflora}. However, on the plasticity side (i.e., acquiring new knowledge), these methods almost universally rely on cross-entropy loss as the default objective for learning from new data.  In this work, we ask whether CE is truly desirable for continual learning. We investigate this question from the perspective of the true classification objective (0–1 loss) and through the lens of entropy.

We begin by reformulating continual learning through a reinforcement learning (RL) lens. The ultimate goal of classification is to minimize the misclassification error (0-1 loss). However, since this loss is non-differentiable,  CE loss has become the de facto surrogate for training deep models, even though the trained models are eventually evaluated on 0-1 loss. We instead cast classification as a one-step Markov Decision Process (MDP), which yields an RL objective equivalent to minimizing misclassification error.  Based on this formulation, we introduce \emph{Expected Policy Gradient} (EPG), a low-variance variant of REINFORCE~\citep{williams1992simple}, to directly optimize 0–1 loss.  

To reveal CE’s limitations, we conduct a comparative analysis of CE and EPG in terms of gradient behavior and entropy dynamics. In parameter space, both share the same gradient direction per sample, but CE implicitly reweights samples—prioritizing harder, high-surprisal cases. This encourages adaptation but risks destabilizing prior knowledge, as the model must correct large prediction errors for those hard samples. By contrast, EPG is more \textit{exploitative}, emphasizing easier samples that already align with the model’s predictions. In the action space, EPG consistently produces lower-entropy output distributions than CE. We hypothesize that CE’s excessive exploration can cause significant deviation from pretrained weights, thereby exacerbating forgetting. This insight highlights the need to carefully balance exploration and exploitation in continual fine-tuning.  

To this end, we propose an \emph{adaptive entropy annealing strategy} (aEPG), which interpolates between CE (exploration) and EPG (exploitation) through a time-dependent weighting scheme. aEPG begins with CE to encourage plasticity and gradually shifts toward EPG to preserve stability. Empirically, aEPG consistently improves performance across four CL benchmarks and multiple PEFT architectures (LoRA, Adapter, and Prefix) (see Table~\ref{tab:sota} and ~\ref{tab:lae}).

Finally, we broaden our study to entropy regularization in continual fine-tuning. While prior work advocates high-entropy training (e.g., label smoothing, focal loss, confidence penalties) for classification, we observe that such approaches harm class-incremental learning with pretrained models. In contrast, low-entropy objectives consistently improve continual adaptation (Table~\ref{tab:focal_all}),   
suggesting that exploitative learning is more effective than aggressive exploration for continual learning with foundation models.

Our contributions are summarized as follows:   
\begin{itemize}
\item We introduce EPG, a policy gradient method that directly optimizes 0–1 loss instead of surrogate objectives like CE.
\item We provide theoretical and empirical evidence contrasting CE’s exploration bias with EPG’s exploitative bias in parameter space and action space.
\item We propose an adaptive entropy annealing strategy (aEPG) that balances exploration and exploitation, improving class-incremental learning.  
\item
We conduct, to our knowledge, the first systematic investigation into entropy regularization for continual learning with PEFT,  uncovering the critical role of entropy reduction.
\end{itemize}
}

\section{Related Work}
\textbf{Continual Learning and Parameter-Efficient Fine-tuning
(PEFT)} Traditional continual learning (CL) methods address forgetting in the train-from-scratch setting through strategies such as penalizing changes to important parameters via regularization losses~\citep{li2017learning,kirkpatrick2017overcoming}, constraining optimization with orthogonal gradients~\citep{farajtabar2020orthogonal, chaudhry2020continual,guo2022adaptive}, replaying past samples~\citep{buzzega2020dark,zhang2022simple}, or isolating parameters for different tasks~\citep{mallya2018packnet,lee2023look,wang2023rehearsal}.  
With the advent of large pretrained transformers, PEFT techniques have achieved state-of-the-art CL performance by restricting the number of trainable parameters. Early advances in continual PEFT, such as L2P, DualPrompt, and CodaPrompt~\citep{wang2022learning,wang2022dualprompt,smith2023coda}, introduced learnable prompt parameters maintained in memory. These methods optimize prompts to guide model predictions while explicitly managing task-invariant and task-specific knowledge.  More recent work has proposed unified frameworks, which integrate adapters~\citep{sung2022vl}, LoRA~\citep{hu2022lora}, and prefix tuning~\citep{DBLP:conf/iclr/LeNNTLH25}, as well as ensemble-based methods like LAE~\citep{gao2023unified} that combine fast and slow learning PEFT experts, and specialized LoRA initialization techniques like InfLoRA~\citep{liang2024inflora}, which design orthogonal subspace constraints
with LoRA to mitigate task interference. 

 \textbf{RL for Pretrained Models}. RL has become a key approach for aligning large pretrained models with human preferences~\citep{casperopen}. Typically, this requires carefully designed reward signals. Most existing work focuses on language models, where external human preference data is treated as the reward signal, and the model is optimized as a policy using policy-gradient methods such as PPO~\citep{ouyang2022training}. Very few studies explore vision models. To the best of our knowledge, only one work applies RL to vision models and is tailored to segmentation and object detection~\citep{pinto2023tuning} with specially designed rewards. In this work, we take a step back and revisit the RL framework for the classification problem to directly minimize 0-1 loss.

\textbf{RL for Continual Learning}
Reinforcement learning has also been applied to improve continual learning performance, mostly for hyperparameter optimization.  ~\cite{xu2018reinforced} employs RL to dynamically select optimal neural architectures for incoming tasks, while ~\cite{zhang2022simple} and ~\cite{liu2023online} introduce a multi-armed bandit framework to adjust hyperparameters like learning rate, training iterations, etc. In contrast, our work uses RL to directly optimize model parameters instead of hyperparameters. %We aim to address this gap in this work. 

\textbf{Entropy Regularization.} 1) \textit{Increasing entropy:} In reinforcement learning, entropy regularization encourages exploration and prevents premature convergence to suboptimal policies through approaches such as intrinsic rewards with entropy bonuses, Boltzmann exploration, and parameter noise injection~\citep{chung2024parseval,haarnoja2018soft,han2023entropy}. In supervised learning, it mitigates overconfident predictions by promoting high-entropy output distributions, improving calibration. A common technique is the confidence penalty~\citep{DBLP:conf/iclr/PereyraTCKH17}, which subtracts a weighted entropy term from the loss to produce more balanced predictions. Subsequent work~\citep{meister2020generalized} unifies label smoothing and confidence penalties, comparing their effectiveness in language generalization tasks, while \citep{mukhoti2020calibrating} shows that focal loss implicitly increases entropy, also enhancing model calibration.  2) \textit{Decreasing entropy:} Entropy minimization is commonly used to deal with unlabeled data. In semi-supervised learning, entropy minimization~\citep{grandvalet2004semi} encourages classifiers to make confident predictions on unlabeled samples, thereby shaping decision boundaries to pass through low-density regions of the input space. More recently, entropy minimization has been leveraged in test-time adaptation~\cite{DBLP:conf/iclr/WangSLOD21}, where models adapt to distribution shifts by enforcing confident, low-entropy predictions on new domains without labels. To our knowledge, we present the first systematic investigation of entropy regularization in continual learning with pretrained models.

\textbf{Direct Minimization of 0-1 Loss}. Direct optimization of 0-1 loss has been studied in supervised classification via approximations and alternative formulations. For example, ~\cite {hasan2019new} proposes a smooth approximation using the posterior mean of a generalized Beta-Bernoulli distribution, and \cite{karpukhin2024exact}  models predictions as a multivariate normal distribution and solves it via orthant integration of the probability density function. 
To the best of our knowledge, our work is the first to approach 0-1 loss optimization from a reinforcement learning perspective.

\textbf{Bandit Multiclass Classification}. Our work differs from multi-classification with bandit information~\citep{kakade2008efficient}. In the bandit feedback setup, the learner does not observe the true label for a given input but only receives binary feedback indicating whether its predicted label is correct.  Notably, this is typically studied in an online setting, and the main objective is to minimize the regret and to achieve sublinear regret~\citep{erez2024real,erez2024fast,crammer2013multiclass}. Our work follows the standard supervised learning with the true labels available.

\section{Methodology}

\subsection{Classification MDP Formulation}
\label{sec:classification_mdp}
We formulate classification and continual fine-tuning as a one-step MDP: the input samples $x \sim d(x)$ form the state space with state distribution $d(x)$ and classification labels constitute the action space $\mathcal{A}$, with a reward function $r\sim\mathcal{R}_{x,a}$ indicating whether an action (predicted label) matches the ground-truth label for $x$, or not. Episodes terminate after one step. The policy $\pi_\theta(a|x)$ is parameterized by a deep neural network. The objective is to maximize expected reward over the policy: 
\begin{equation}
\label{eq:rl_objective}
\begin{aligned}
    J_\pi(\theta)& = \mathbb{E}_{x\sim d(x),a\sim \pi_\theta(a|x)}[r] \\
    &= \sum_{x \in \mathcal{X}} d(x) \sum_{a \in \mathcal{A}} \pi_\theta(a|x) \mathcal{R}_{x,a}.
\end{aligned}
\end{equation}
More specifically, we define a deterministic reward function based on ground-truth labels:

\begin{equation}
\label{eq_reward_def}
\mathcal{R}_{x,a} = 
\begin{cases} 
1, & \text{if } a = y \\
0, & \text{otherwise}
\end{cases}
\end{equation}

This reward scheme assigns a value of one for correct classifications and zero otherwise, directly aligning the reinforcement learning objective with the goal of maximizing classification accuracy.

We consider supervised classification, where ground truth labels are available during training. Note that the reward model is \textit{fully observable to the learning agent} when the class label is given. This makes the proposed classification MDP differ from the problem setting of classification under bandit feedback~\cite{kakade2008efficient}.
%The agent receives immediate feedback about prediction correctness, directly aligning the reinforcement learning goal with standard classification objectives. 

\textbf{Connection between RL Objective and 0-1 Loss}. %We establish the relationship between the reinforcement learning objective and the 0-1 classification loss. 
For a classifier $h_\theta$ with true labels $y$ and predictions $h_\theta(x)$, the \textit{0-1 loss} is defined as:
\begin{equation}
    \mathcal{L}_{01}(y, h_\theta(x)) =
    \begin{cases}
        0, & \text{if } h_\theta(x) = y \text{ (correct class)} \\
        1, & \text{if } h_\theta(x) \neq y \text{ (incorrect class)}
    \end{cases}
\end{equation}

Building upon the RL objective in Eq.~\ref{eq:rl_objective} and the reward function in Eq.~\ref{eq_reward_def}, we derive the following connection:

\begin{proposition}
\label{proposition_01}
Minimizing the 0-1 loss of classifier $h_\theta$ is equivalent to maximizing the RL objective:
\begin{equation}
    \min_\theta \mathcal{L}_{01}(h_\theta) = \max_\theta J_h(\theta)
\end{equation}

\end{proposition}
\begin{proof}
    See Appendix~\ref{sec:appendix_proof_01}
\end{proof}
% \begin{proof}
% By interpreting $h_\theta(x)$ as the policy $\pi_\theta(a|y)$ in Eq~\ref{eq:rl_objective} and applying a constant baseline of value 1 to the reward function $\mathcal{R}_{x,a}$ (Eq.~\ref{eq_reward_def}), we obtain:
% \begin{equation}
% \begin{aligned}
%        J_h(\theta) &=\mathbb{E}_{x\sim d(x),a\sim h_\theta}[r]\\
%       & = 1-\sum_{x \in \mathcal{X}} d(x) \sum_{a \in \mathcal{A}} h_\theta(a|x) (-\mathcal{R}_{x,a}+1)\\
% &= 1 - \mathcal{L}_{01}(h_\theta) .
% \end{aligned}
% \end{equation}
% The constant offset does not affect the optimization objective, thus establishing the equivalence.
% \end{proof}

%The 0-1 loss presents fundamental challenges for gradient-based optimization due to its discontinuous and non-differentiable nature. We address this limitation through a reinforcement learning perspective that reformulates classification as policy optimization.  
Proposition 1 demonstrates that 0-1 loss minimization can be addressed from an RL  perspective that reformulates classification as policy optimization.

While conventional classification approaches typically implement a deterministic mapping $h_\theta: \mathcal{X} \rightarrow \mathcal{Y}$ (which could alternatively be viewed as a deterministic policy in the proposed framework and optimized via deterministic policy gradient methods~\cite{silver2014deterministic}), this paper instead explores a stochastic policy with softmax parameterization: $\pi_\theta(a|x)=e^{f_\theta(a|x)}/\sum_k e^{f_\theta(k|x)}$, where $f_\theta: \mathcal{X} \rightarrow \mathbb{R}^K$ denotes the model's logit outputs. This parameterization not only maintains the familiar structure of softmax-based classification but also establishes a principled connection between policy gradient optimization and cross-entropy minimization. Through this formulation, we can directly investigate how policy gradient methods relate to traditional classification objectives (CE) while handling the non-differentiable 0-1 loss, as shown in the next section.
%In the following sections, we investigate how policy gradient methods can implicitly optimize the 0-1 loss in section~\ref{sec:epg} and analyze this relationship with cross entropy loss~\ref{sec:ce_connection}.

\subsection{Expected Policy Gradient}
\label{sec:epg}

We solve the RL problem described above using policy gradient methods. While traditional approaches such as REINFORCE~\citep{williams1992simple} and PPO~\citep{schulman2017proximal} rely on stochastic action sampling, we derive a more efficient gradient estimator by exploiting the inherent structure of the classification MDP.

Based on the policy gradient theorem, the gradient of Eq~\ref{eq:rl_objective} can be computed using the likelihood ratio gradient estimator~\citep{sutton1999policy}. For one-step MDPs with immediate rewards, we have:

\begin{equation}
    \nabla_\theta J(\theta) = \mathbb{E}_{x\sim d(x), a\sim \pi_\theta(a|x)} \left[
    \mathcal{R}_{x,a} \nabla_\theta \log \pi_\theta(a|x) \right]
\end{equation}

The REINFORCE policy gradient algorithm~\cite{williams1992simple} approximates this expectation through Monte Carlo sampling. Given the sampled trajectories $\{x_i, a_i, r_i\}_{N}$, the gradient can be estimated as:

\begin{equation}
    \hat{g}_{\text{REINFORCE}} = \frac{1}{N}\sum_{x_i\sim d(x), a_i\sim \pi_\theta} \mathcal{R}_{x_i,a_i} \nabla_\theta \log \pi_\theta(a_i|x_i)
\end{equation}

This type of sampling-based policy gradient method, as employed by REINFORCE and Proximal Policy Optimization (PPO)~\citep{schulman2017proximal}, is widely used in deep reinforcement learning tasks and for fine-tuning large language models with human feedback. However, we observe that the sampling-based approach does not exploit the simplicity of classification tasks.
Crucially, in our classification MDP formulation, the reward function is available to the learner, since the reward $\mathcal{R}_{x,a}$ for all actions is available once the class label for a sample is given (see Eq.~\ref{eq_reward_def}).
This allows us to compute the expectation over actions exactly in the gradient estimator while only sampling from the state distribution $d(x)$:
\begin{equation}
\label{eq:grad_epg}
\begin{aligned}
    \hat{g}_{\text{EPG}} 
    % &= \sum_{x_i\sim d(x)} \mathbb{E}_{a\sim \pi_\theta(a|x_i)} \left[ \mathcal{R}_{x_i,a} \nabla_\theta \log \pi_\theta(a|x_i) \right] \\
    &= \frac{1}{N}\sum_{x_i\sim d(x)} \sum_{a \in \mathcal{A}} \pi_\theta(a|x_i) \mathcal{R}_{x_i,a} \nabla_\theta \log \pi_\theta(a|x_i)
\end{aligned}
\end{equation}
We term this the \textit{Expected Policy Gradient} (EPG) to distinguish it from methods that need to sample actions.  As EPG uses the exact expectation, it can eliminate the noise caused by action sampling. In other words, it maintains the true gradient's expectation while reducing variance in gradient estimate, i.e., $\text{Var}[\hat{g}_{\text{EPG}}] \leq \text{Var}[\hat{g}_{\text{REINFORCE}}]$, and $\mathbb{E}[\hat{g}_{\text{EPG}}] = \mathbb{E}[\hat{g}_{\text{REINFORCE}}]$.
%{\color{blue}The variance reduction property follows directly from the law of total variance, as $\mathbb{V}[\hat{g}_{\text{EPG}}] \leq \mathbb{V}[\hat{g}_{\text{REINFORCE}}]$, while $\mathbb{E}[\hat{g}_{\text{EPG}}] = \mathbb{E}[\hat{g}_{\text{REINFORCE}}]$.}
% From the perspective of reinforcement learning, we can see that cross entropy loss leads to a bootstrapped version of policy gradient, where the reward is bootstrapped in the following way. Hard samples with low $\pi_\theta(y|x)$ will receive higher reward compared to easier samples.

% \begin{equation}
% \label{eq_reward_ce}
% \mathcal{R}^{\text{CE}}_{x, a}  =\frac{1}{\pi_\theta(a|x)}\mathcal{R}^{\text{EPG}}_{x, a}
% \end{equation}

\subsection{Comparative Analysis of CE and EPG}
\begin{figure}[t]
    \centering
    \subfigure[Split-Food101-10]{
    \includegraphics[width=0.7\linewidth]{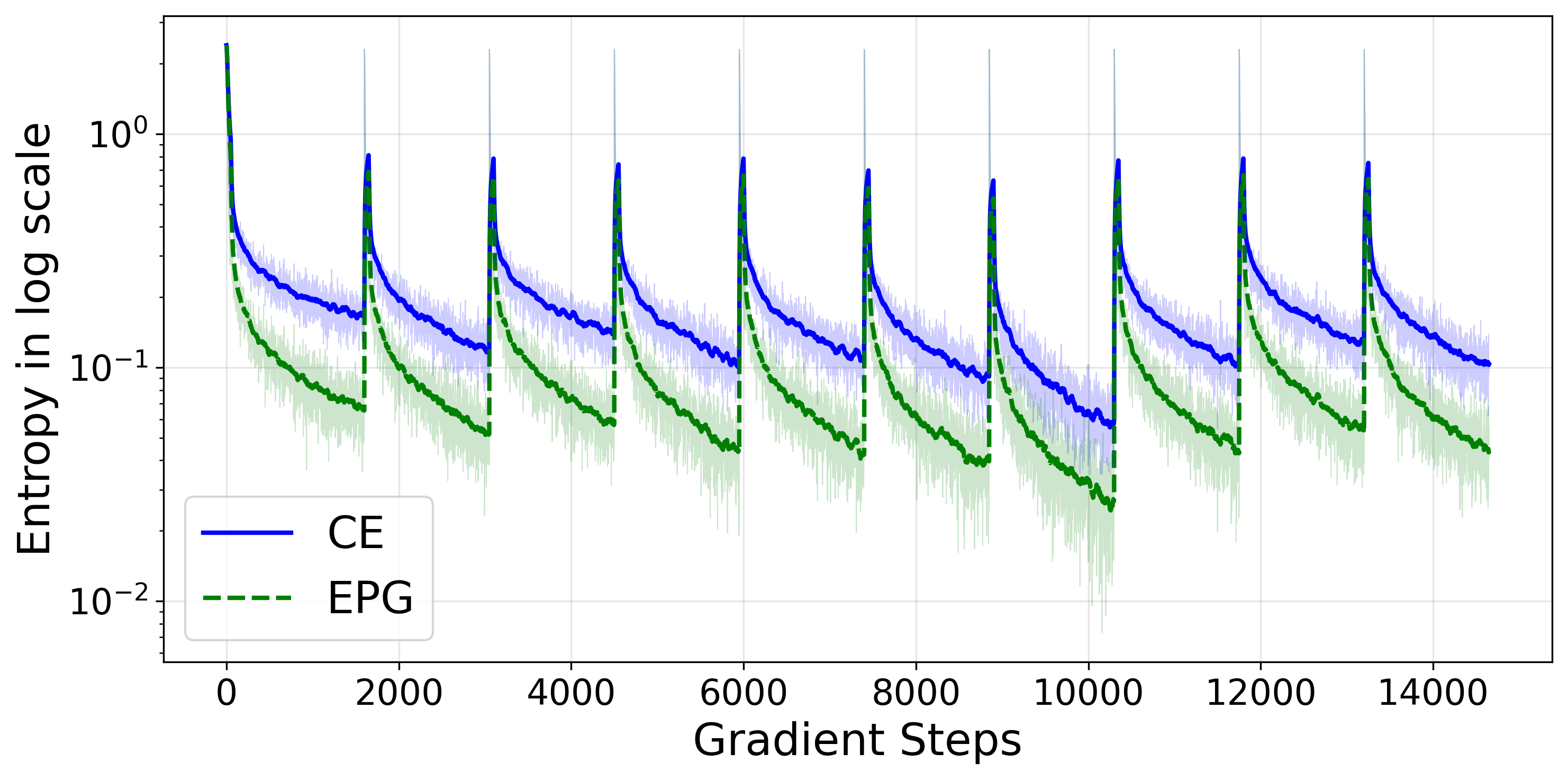}}
    \subfigure[Split-ImageNetR200-10]{
     \includegraphics[width=0.7\linewidth]{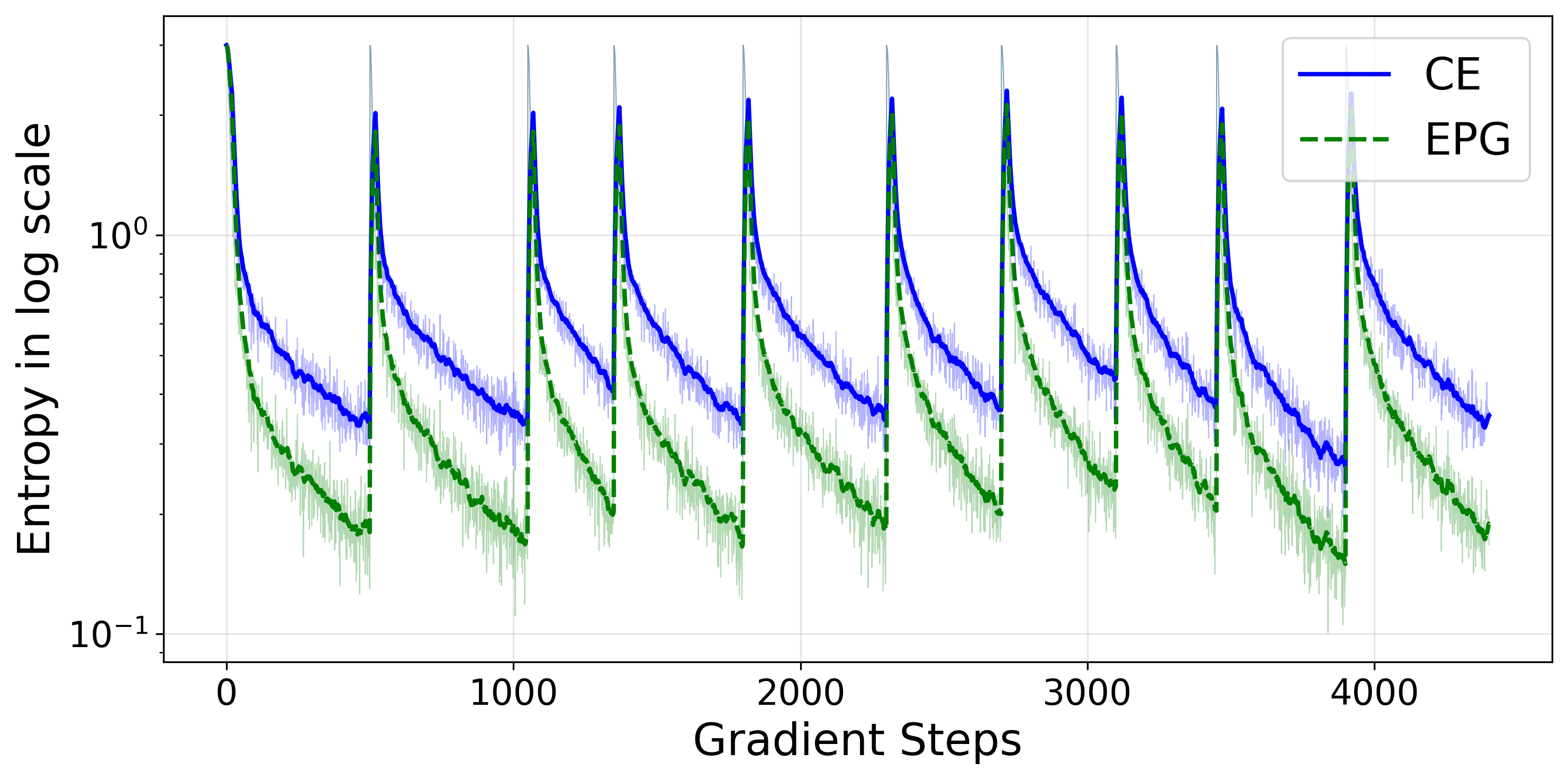}}
    \caption{The entropy of the output distribution of the model: EPG leads to higher entropy than CE. }
    \label{fig:epg_entropy}
\end{figure}
\label{sec:ce_connection}

\label{sec:over_exploration}
{

\textbf{Gradient Analysis: Sample Weighting}.  We first study the connection between EPG and CE optimization by examining their gradient relationship. Given the target distribution $q(a|x)$ and the softmax output $\pi_\theta(a|x)$, the gradient of CE (Eq.~\ref{eq:grad_ce}) is:

\begin{equation}
\label{eq:grad_ce}
\hat{g}_{\textbf{CE}} = -\frac{1}{N}\sum_{x\sim d(x)}\sum_a q(a|x) \nabla_\theta \log\pi_\theta(a|x).
\end{equation}

Note that both EPG and CE gradients involve $\nabla_\theta \log \pi_\theta(a|x)$. For one-hot labels (i.e., when $q(a|x)$ follows Dirac delta distribution), the gradient for a sample $(x_i, y_i)$ simplifies to $\hat{g}^i_{\text{CE}}(x_i,y_i) = -\nabla_\theta \log\pi_\theta(y_i|x_i)$. Comparing this with Eq.~\ref{eq:grad_epg}, we derive:

% \begin{equation}
% \label{eq:grad_relation}
% \begin{aligned}
% \hat{g}^i_{\text{EPG}}(x_i,y_i) &= \pi_\theta(y_i|x_i) \nabla_\theta \log\pi_\theta(y_i|x_i) \\
% &= -\pi_\theta(y_i|x_i) \hat{g}^i_{\text{CE}}(x_i,y_i).    
% \end{aligned}
% \end{equation}

\begin{equation}
\label{eq:grad_relation}
\begin{aligned}
\hat{g}^i_{\text{CE}}(x_i,y_i) &= -\frac{1}{\pi_\theta(y_i|x_i)}  \hat{g}^i_{\text{EPG}}(x_i,y_i).    
\end{aligned}
\end{equation}

This reveals that EPG and CE produce gradients in the \textit{same direction} but with \textit{different sample weightings}. Specifically, CE assigns larger weights to uncertain predictions (where $\pi_\theta(y_i|x_i)$ is close to 0), whereas EPG emphasizes confident predictions (where $\pi_\theta(y_i|x_i)$ is close to 1). 
%Confidence-based weighting has also been explored in prior work such as focal loss~\citep{lin2017focal}. However, focal loss incorporates weighting at the \textit{loss level}. In contrast, EPG applies confidence-based weighting \textit{directly in the gradient}, which more explicitly governs parameter updates. 

The CE loss is known to be vulnerable to label noise. Here, we offer a conceptual explanation by comparing its gradient behavior with 0-1 loss optimization. 
Equation~\ref{eq:grad_relation} shows that CE disproportionately weights hard samples with larger gradient magnitudes. Although this focus can accelerate learning, it increases susceptibility to overfitting on noisy or mislabeled examples. In contrast, EPG—which directly optimizes the 0-1 loss—assigns higher weights to confident predictions during gradient updates, promoting the learning of general patterns and consequently enhancing robustness.

\begin{table*}[htbp]
%\label{tab:sota}
\centering
\caption{The performance of continual PEFT methods on Split-ImageNet-R (averaged over 5 runs).
}
\label{tab:sota}
 \vspace{4pt}
\resizebox{0.9\textwidth}{!}{
\begin{tabular}{lcccccccc}
\toprule
Tasks &  \multicolumn{2}{c}{\textbf{5 Tasks}} & \multicolumn{2}{c}{\textbf{10 Tasks}} & \multicolumn{2}{c}{\textbf{20 Tasks}} \\
% \cmidrule(lr){4-5} \cmidrule(lr){6-7} \cmidrule(lr){8-9}
\cmidrule(lr){2-3} \cmidrule(lr){4-5} \cmidrule(lr){6-7}
Methods  & $A_{5} $ & $\tilde{A}_{5}$   & $A_{10} $ & $\tilde{A}_{10}$ & $A_{20} $ & $\tilde{A}_{20}$ \\
\midrule
% LwF  &76.3 $\pm$	0.2	& 80.4 $\pm$ 0.2	& 75.0  $\pm$	0.2	& 79.7 $\pm$ 0.2 & 73.3  $\pm$ 0.3	& 78.6 $\pm$ 0.1
% \\
L2P  & 71.1 $\pm$	0.4	& 77.0	$\pm$ 0.7 &	69.3 $\pm$	0.3 & 76.8 $\pm$	0.7	 & 65.6 $\pm$ 0.3 &	74.2 $\pm$	1.0
 \\
DualPrompt  & 72.9 $\pm$ 0.3 & 76.0 $\pm$ 0.3 & 71.2 $\pm$ 0.1 & 75.4 $\pm$ 0.1 & 71.2 $\pm$ 0.1 & 74.8 $\pm$ 0.1 \\
Coda-Prompt & 76.5 $\pm$	0.5 &	81.8 $\pm$	0.7 &	75.0 $\pm$	0.4 &	81.6 $\pm$	0.7 &	71.5 $\pm$	0.4 &	79.0 $\pm$	1.0
  \\
InfLoRA  & 76.8 $\pm$ 0.4 & 80.8 $\pm$ 0.3 & 74.2 $\pm$ 0.1 & 79.5 $\pm$ 0.2 & 68.6 $\pm$ 0.5	& 74.8 $\pm$	0.4
 \\ \midrule
LoRA & 74.8 $\pm$ 0.0 & 79.8 $\pm$ 0.1 & 74.3 $\pm$ 0.1 & 79.2 $\pm$ 0.3 & 73.2 $\pm$ 0.1 & 78.7 $\pm$ 0.0 \\
LoRA + aEPG  & 77.2 $\pm$ 0.0 & {81.9 $\pm$ 0.0} & {75.8 $\pm$ 0.3} & {80.9 $\pm$ 0.0 }& {74.1 $\pm$ 0.2} &{79.7 $\pm$ 0.2} \\ \midrule
LwfLoRA  &76.3 $\pm$	0.2	& 80.4 $\pm$ 0.2	& 75.0  $\pm$	0.2	& 79.7 $\pm$ 0.2 & 73.3  $\pm$ 0.3	& 78.6 $\pm$ 0.1
\\ 
LwfLoRA + aEPG &  78.1 $\pm$ 0.2 &	81.8  $\pm$	0.1 & 76.2 $\pm$	0.2	& 80.8 $\pm$	0.2 & 74.3 $\pm$	0.2	& 79.7 $\pm$	0.1 \\ \midrule
LAE &  76.1 $\pm$ 0.2 & 80.6 $\pm$ 0.1 & 75.4 $\pm$ 0.0 & 79.9 $\pm$ 0.3 & 73.9 $\pm$ 0.3 & 79.2 $\pm$ 0.1 \\
LAE  + aEPG    & \textbf{78.3 $\pm$ 0.1} & \textbf{82.3 $\pm$ 0.0} & \textbf{76.7 $\pm$ 0.3} & \textbf{81.4 $\pm$ 0.2} & \textbf{75.0 $\pm$ 0.3} & \textbf{80.0 $\pm$ 0.1 }\\
\bottomrule
\end{tabular}
}
\end{table*}

\begin{table*}[htbp]
    %\label{tab:lae}
    \centering
    \caption{The effectiveness of aEPG across different datasets and PEFT modules (averaged over 3 runs).}
    \vspace{4pt}
    \resizebox{0.9\textwidth}{!}{
        \begin{tabular}{ll>{\centering\arraybackslash}m{1.5cm}cccccccc}
        \toprule
    \multirow{2}{*}{PEFT} &\multirow{2}{*}{Algo}   & \multicolumn{2}{c}{Split-Food101} & \multicolumn{2}{c}{Split-CUB200} & \multicolumn{2}{c}{CLRS25} \\
            \cmidrule(lr){3-4} \cmidrule(lr){5-6} \cmidrule(lr){7-8} \cmidrule(lr){9-10}
        &    & $A_{10} $ & $\tilde{A}_{10}$ & $A_{10} $ & $\tilde{A}_{10}$ & $A_{5} $ & $\tilde{A}_{5}$ \\
        \midrule
        %\textbf{Lora} \\
\multirow{2}{*}{LoRA} &
LAE  & $83.2 \pm 0.1$ & $88.8 \pm 0.0$ & $82.7 \pm 0.3$ & $85.8 \pm 0.3$ & $73.0 \pm 0.6$ & $82.9 \pm 0.3$ \\
& LAE+aEPG  & $84.7 \pm 0.0$ & $89.7 \pm 0.1$ & $\mathbf{84.1 \pm 0.2}$ & $\mathbf{86.3 \pm 0.2}$ & $\mathbf{74.8 \pm 0.5}$ & $\mathbf{84.5 \pm 0.4}$ \\
\midrule
\multirow{2}{*}{Adapter} &  LAE  & $83.6 \pm 0.2$ & $88.9 \pm 0.1$ & $82.2 \pm 0.1$ & $85.7 \pm 0.2$ & $70.3 \pm 1.7$ & $81.0 \pm 1.0$ \\
& LAE+aEPG  & $\mathbf{85.0 \pm 0.1}$ & $\mathbf{89.7 \pm 0.1}$ & $83.9 \pm 0.2$ & $86.3 \pm 0.1$ & $74.1 \pm 0.5$ & $83.5 \pm 0.3$ \\
\midrule
\multirow{2}{*}{Prefix} & LAE  & $82.9 \pm 0.1$ & $88.7 \pm 0.1$ & $80.7 \pm 0.1$ & $84.5 \pm 0.1$ & $68.8 \pm 0.7$ & $80.0 \pm 0.5$ \\
& LAE+aEPG & $84.4 \pm 0.1$ & $89.4 \pm 0.1$ & $81.0 \pm 0.3$ & $84.5 \pm 0.3$ & $71.9 \pm 0.7$ & $82.0 \pm 0.5$ \\ \bottomrule
    \end{tabular}
    }
    \label{tab:lae}
\end{table*}

\textbf{Objective Analysis: Noise Tolerance}. The robustness of EPG to label noise can also be established theoretically. The RL framework underlying EPG is inherently designed to handle stochastic reward signals, making it naturally tolerant to noise. Proposition~\ref{proposition:label noise} demonstrates that the EPG objective exhibits noise tolerance under symmetric label noise conditions.

\begin{proposition}
\label{proposition:label noise}
In a $K$-class classification problem, EPG is noise-tolerant under symmetric label noise provided that the noise rate $\eta$ satisfies $\eta < 1 - \tfrac{1}{K}$.
\end{proposition}
\begin{proof}
See Appendix~\ref{sec:appendix_proof_label_noise}.
\end{proof}

\textbf{Entropy Analysis: Exploration vs. Exploitation}. Our previous analysis examines how the sample weighting scheme influences gradient optimization in parameter space. To understand how these parameter changes translate to the action space (i.e., the output of the deep neural network), we analyze the evolution of output distribution entropy during continual fine-tuning.
 As shown in Figure~\ref{fig:epg_entropy}, when learning each new task, the model's predictions are initially nearly random, resulting in high entropy. As training progresses, this entropy gradually decreases. Perhaps surprisingly, we observe that EPG reduces entropy significantly faster than CE and achieves lower final entropy levels (Fig.~\ref{fig:epg_entropy}), despite EPG having smaller gradient magnitudes than CE ($|\hat{g}^i_{\text{EPG}}(x_i,y_i)|=\pi_\theta(y_i|x_i) |\hat{g}^i_{\text{CE}}(x_i,y_i)|\leq |\hat{g}^i_{\text{CE}}(x_i,y_i)|$).

Beyond empirical observations of entropy dynamics, Proposition~\ref{proposition:KL} offers theoretical insight under symmetric label noise, showing that the RL objective underlying EPG inherently reduces entropy while also minimizing the KL divergence between the predicted and target distributions.
\begin{proposition}  
\label{proposition:KL}
For $K$-class classification with symmetric label noise, the EPG objective
satisfies
\[
\max_\theta J^{RL}_{p_\theta}(\theta) \equiv \min_\theta \left[ D_\text{KL}(p_\theta \parallel q_\eta) + H(p_\theta) \right],
\]
where $p_\theta$ is the model's predictive distribution and $q_n$ is the target distribution.
\end{proposition}

% \begin{proposition}  
% \label{proposition:KL}
% For hard-label classification, let $p_\theta$ denote the model's predictive distribution and let $q_\epsilon$ be an $\epsilon$-smoothed target distribution, i.e.,
% \begin{equation}
% q_\epsilon(y_k|x) = 
% \begin{cases}
% 1 - \epsilon & \text{if } y_k = y^*, \\
% \frac{\epsilon}{K-1} & \text{otherwise},
% \end{cases}
% \end{equation}
% where $y^*$ is the correct class and $K$ is the number of classes. Then, in the limit $\epsilon \to 0$, the reinforcement learning objective
% \begin{equation}
% J^{RL}_{p_\theta}(\theta) = \mathbb{E}_{y \sim p_\theta}[\mathcal{R}(y,x)], \quad 
% \mathcal{R}(y,x) = \mathbf{1}_{y = y^*},
% \end{equation}
% satisfies
% \begin{equation}
% \max_\theta J^{RL}_{p_\theta}(\theta) \equiv \lim_{\epsilon \to 0} \min_\theta \left[ D_\text{KL}(p_\theta \parallel q_\epsilon) + H(p_\theta) \right].
% \end{equation}
% \end{proposition}

% \begin{proposition}  
% \label{proposition:KL}
% For hard-label classification, the reinforcement learning objective satisfies:
% \begin{equation}
%     \max_\theta J^{RL}_{p_\theta}(\theta) \equiv \min_\theta \left[ D_\text{KL}(p_\theta \parallel q) + H(p_\theta) \right],
% \end{equation}
% where $p_\theta$ is the model's predictive distribution and $q$ is the target distribution. 
% %This establishes that Expected Policy Gradient optimization simultaneously minimizes the KL divergence between predictions and targets and reduces the entropy of the output distribution

% \end{proposition}

\begin{proof}See Appendix~\ref{sec:appendix_proof_kl}
%[Proof Sketch]The equivalence follows from 1) decomposing the RL objective using the baseline subtraction technique from policy gradient methods; 2) identifying the entropy and divergence terms through algebraic manipulation
\end{proof}

Proposition~\ref{proposition:KL} establishes an connection between the 0-1 loss and CE loss under noisy label setting: while the CE loss explicitly minimizes the difference between the target and predicted distributions (via minimizing $D_\text{KL}(q_\eta||p_\theta)$), the 0-1 loss not only reduces this distributional disparity (via minimizing $D_\text{KL}(p_\theta||q_\eta)$) but also implicitly minimizes entropy. %This dual optimization mechanism provides a theoretical explanation for the empirical observation that EPG drives the model toward lower-entropy solutions compared to CE.

In summary, the entropy and gradient analysis demonstrate a key difference in their optimization behaviors: CE exhibits exploratory behavior while  EPG demonstrates exploitative tendencies. 
Concretely, in parameter space, CE focuses on hard examples that the model is currently uncertain about, which yield high information gain (or ``surprisal'' with $\pi_\theta(y_i|x_i)$ close to 0), as the model must correct large prediction errors. While this drives adaptation to a new task, it can also disrupt existing knowledge, increasing forgetting. In contrast, EPG/aEPG focuses on easier samples that already align well with the model’s predictions (with $\pi_\theta(y_i|x_i)$ close to 1), resulting in lower surprisal and less disruptive updates. In the action space, CE’s exploration bias manifests as higher output entropy than 0-1 loss optimization.

While prior work has shown that increased entropy can benefit classification models by promoting exploration~\citep{DBLP:conf/iclr/PereyraTCKH17,meister2020generalized,dubey2018maximum}, these advantages have primarily been observed in train-from-scratch settings. We hypothesize that this relationship may fundamentally differ for pretrained models, where excessive exploration could cause substantial deviation from the pretrained weights and disrupt past task knowledge. This necessitates a careful re-examination of the exploration-exploitation tradeoff when continually fine-tuning pretrained models.

}

\subsection{Algorithm: Adaptive Entropy Annealing}
\label{sec:aEPG}
To balance exploration (via CE optimization) and exploitation (via EPG optimization), we propose an adaptive entropy annealing method that combines both objectives with a time-dependent weighting scheme. The gradient update is defined as:
\begin{equation}
\label{eq:aepg}
g_{\text{aEPG}}(\theta) = \alpha_t g_{\text{CE}}(\theta) + (1-\alpha_t)(-g_{\text{EPG}}(\theta)),
\end{equation}
where $\alpha_t \in [0,1]$ is an annealing coefficient that evolves throughout training.  The transition is controlled by a sigmoid schedule:  
$
\alpha_t = \sigma\!\left(\tau\frac{T-2t}{T}\right)$,
where $T$ is the total number of training steps, $\sigma(x) = (1 + e^{-x})^{-1}$ is the sigmoid function, and $\tau$ determines the range and sharpness of the transition. We use $\tau=6$ by default, which smoothly interpolates from $\sigma(6)=0.997 \approx 1$  at initialization to $\sigma(-6)=0.003 \approx 0$ convergence.  This design enables a smooth transition from early exploration to late-stage exploitation. Training begins with pure cross-entropy optimization ($\alpha_0=1$), which maintains high entropy, and gradually shifts toward pure EPG optimization ($\alpha_T=0$), which directly optimizes the 0-1 loss.

\section{Experiments}
%%%%%%%%%%%%%%%%%%%%%%%%%%
%\subsection{Experiment Setup}

	\begin{table*}[htbp]
		\centering
		\caption{The effect of entropy regularization across four datasets using LoRA (averaged over 5 runs). }
         \vspace{2pt}
		\resizebox{0.99\textwidth}{!}{
			\begin{tabular}{l>{\centering\arraybackslash}m{1.5cm}ccccccccc}
			\toprule
		\multirow{2}{*}{Algo} &\multirow{2}{*}{$H(p_\theta)$} & \multicolumn{2}{c}{Split-ImageNet-R} & \multicolumn{2}{c}{Split-Food101} & \multicolumn{2}{c}{Split-CUB200} & \multicolumn{2}{c}{CLRS25} \\
			\cmidrule(lr){3-4} \cmidrule(lr){5-6} \cmidrule(lr){7-8} \cmidrule(lr){9-10}
            	% \cmidrule(lr){4-5} \cmidrule(lr){6-7} \cmidrule(lr){8-9} \cmidrule(lr){10-11}
			  & & $A_{10} $ & $\tilde{A}_{10}$   & $A_{10} $ & $\tilde{A}_{10}$ & $A_{10} $ & $\tilde{A}_{10}$ & $A_{5} $ & $\tilde{A}_{5}$ \\
			\midrule
			%\textbf{Lora} \\
			CE & - & 74.1 $\pm$ 0.4 & 79.3 $\pm$ 0.3 & 83.2 $\pm$ 0.2 & 88.8 $\pm$ 0.2 & 83.3 $\pm$ 0.2 & 86.2 $\pm$ 0.1 & 74.2 $\pm$ 0.8 & 83.9 $\pm$ 0.2 \\ \midrule
			Focal &$\uparrow$ & 72.4 $\pm$ 0.3 & 77.9 $\pm$ 0.3 & 82.7 $\pm$ 0.3 & 88.4 $\pm$ 0.2 & 82.9 $\pm$ 0.3 & 86.1 $\pm$ 0.2 & 73.0 $\pm$ 0.9 & 82.5 $\pm$ 0.4 \\
			LS &$\uparrow$ & 70.6 $\pm$ 0.2 & 76.7 $\pm$ 0.2 & 77.5 $\pm$ 0.4 & 85.2 $\pm$ 0.2 & 82.9 $\pm$ 0.2 & 86.6 $\pm$ 0.2 & 74.8 $\pm$ 1.1 & 85.1 $\pm$ 0.5 \\
			CP &$\uparrow$ & 72.4 $\pm$ 0.8 & 77.9 $\pm$ 0.7 & 83.0 $\pm$ 0.2 & 88.9 $\pm$ 0.2 & 83.3 $\pm$ 0.3 & 86.5 $\pm$ 0.2 & 74.0 $\pm$ 1.0 & 83.9 $\pm$ 0.3 \\ \midrule
			EPG &$\downarrow$ & 75.1 $\pm$ 0.4 & 80.0 $\pm$ 0.2 & 83.5 $\pm$ 0.3 & 88.9 $\pm$ 0.2 & 84.2 $\pm$ 0.1 & 85.9 $\pm$ 0.1 & 74.6 $\pm$ 0.8 & 84.3 $\pm$ 0.7 \\
			aEPG & $\downarrow$& \textbf{75.5 $\pm$ 0.1} & \textbf{80.9 $\pm$ 0.1} & \textbf{84.4 $\pm$ 0.1} & \textbf{89.5 $\pm$ 0.1 }& 84.7 $\pm$ 0.3 & 86.7 $\pm$ 0.1 & \textbf{76.3 $\pm$ 0.4} & \textbf{85.5 $\pm$ 0.3 }\\
			EP &$\downarrow$ & 75.1 $\pm$ 0.2 & 80.4 $\pm$ 0.3 & 84.0 $\pm$ 0.2 & 89.2 $\pm$ 0.2 & \textbf{85.0 $\pm$ 0.2} & \textbf{87.2 $\pm$ 0.1} & 74.8 $\pm$ 0.8 & 84.6 $\pm$ 0.5 \\
			\bottomrule
		\end{tabular}
	}
      \label{tab:focal_all}
	\end{table*}

\begin{figure*}
    \centering
\subfigure[Incremental test accuracy]{\includegraphics[width=0.35\linewidth]{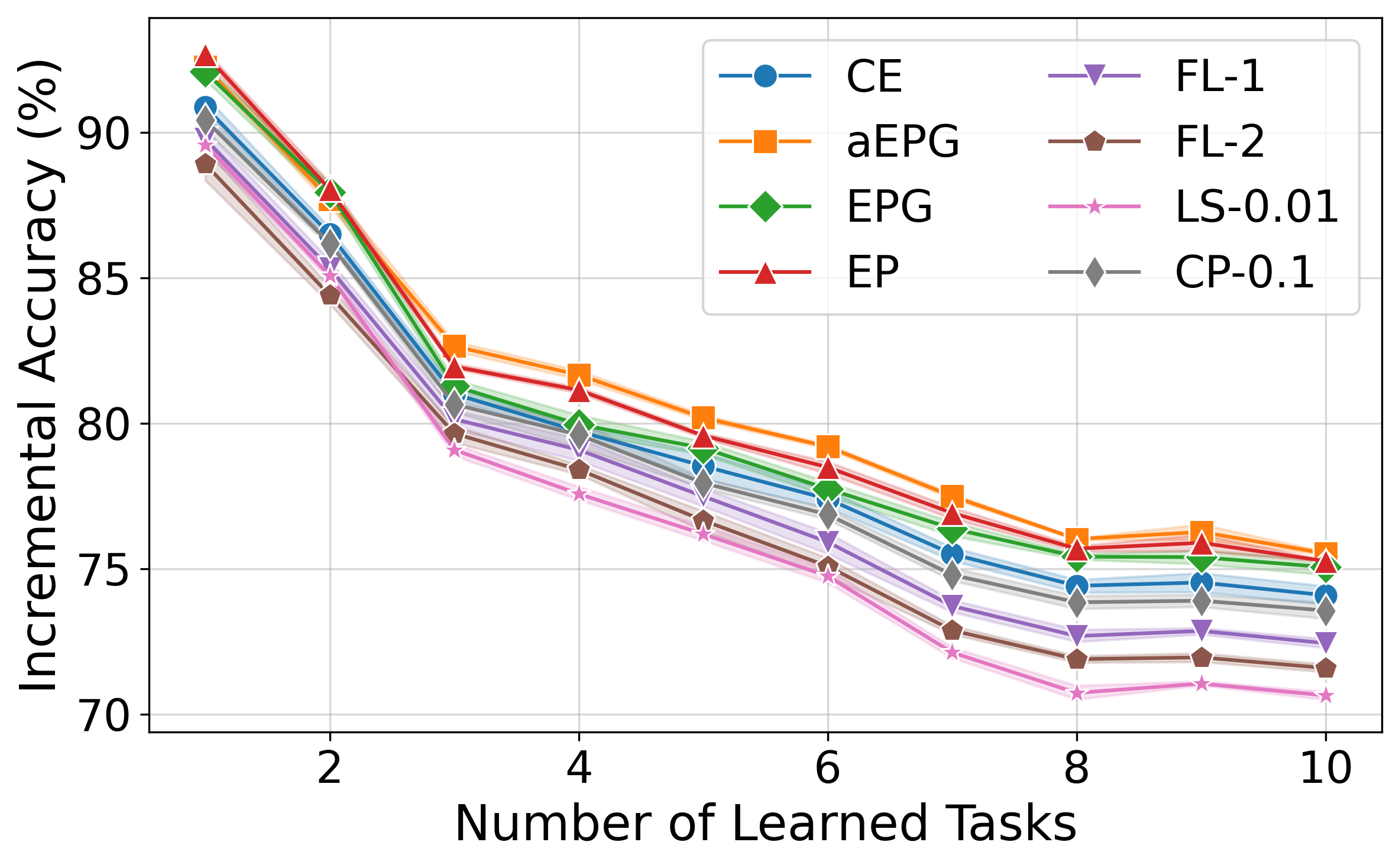}}
\subfigure[Entropy of output distribution]{\includegraphics[width=0.42\linewidth]{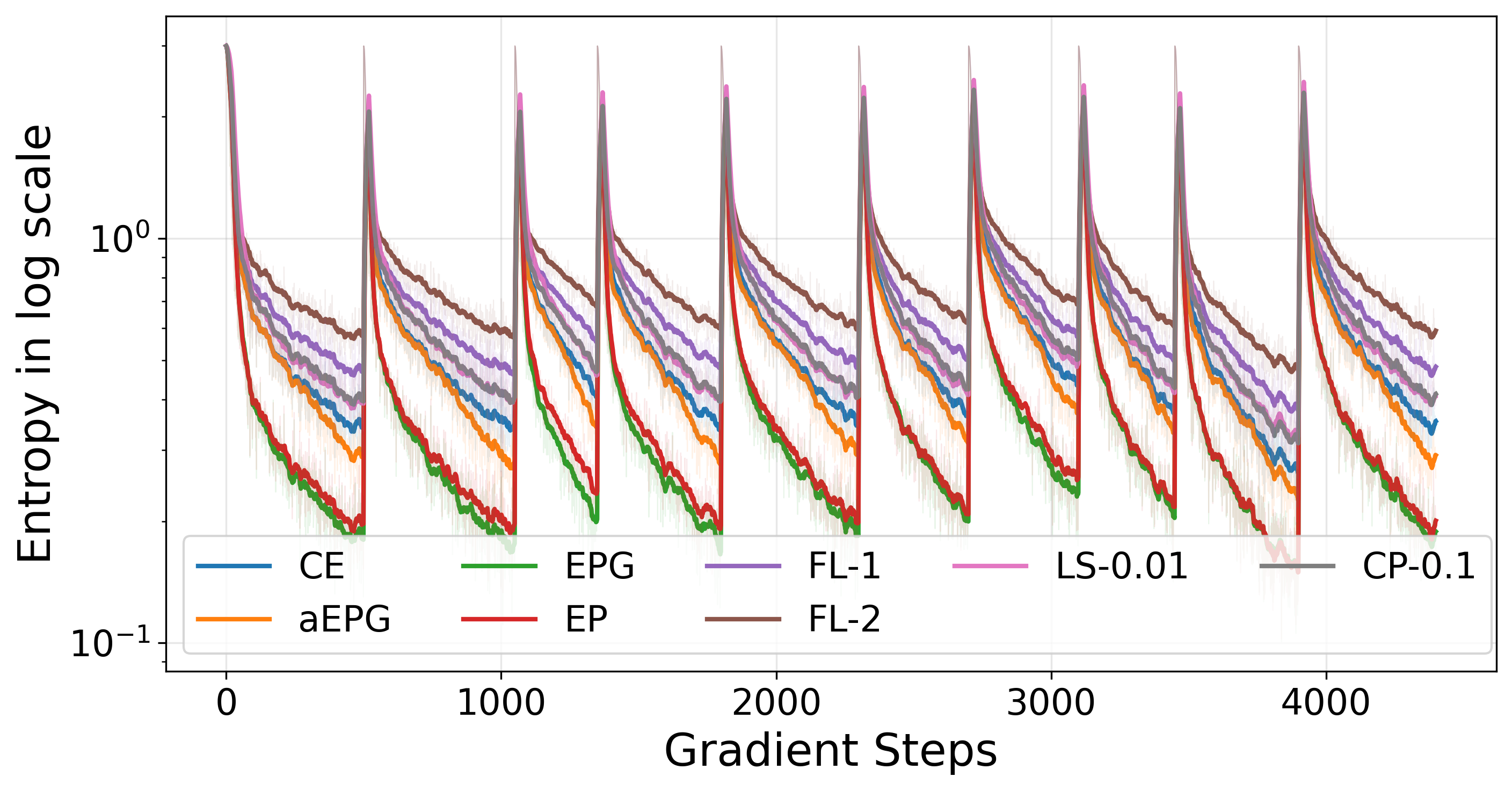}}
    \caption{Entropy dynamics on Split-ImageNet-R. Results are averaged across 5 independent runs, with the 95\% confidence interval shaded. Entropy-reducing methods (EPG, aEPG, EP) achieve higher accuracy than the cross-entropy (CE) baseline, while entropy-increasing methods (focal loss, label smoothing, confidence penalty) result in lower performance. Corresponding results for the Split-Food101 dataset are shown in Appendix Fig.~\ref{fig:focal_food101}.}
    \label{fig:focal_imagenet-r}
\end{figure*}

\textbf{Datasets and Evaluation Metrics.} We evaluate our approach on four diverse datasets: \textbf{ImageNet-R}~\citep{hendrycks2021many} has renditions of 200 ImageNet classes with 24,000 training and 6,000 test samples, naturally exhibiting class imbalance. We partition it into 5, 10, and 20 sequential tasks. \textbf{Food-101}~\citep{bossard2014food} provides balanced classification across 101 food categories (750 training/250 test images per class, 101k total), split into 10 tasks. \textbf{CUB200} ~\citep{wah2011caltech} contains 11,788 bird images across 200 species, which we organize into 10 incremental tasks (20 classes each). \textbf{CLRS}~\citep{li2020clrs} contains large-scale remote sensing data with 25 scene classes (600 images/class, 256×256 resolution), divided into 5 sequential tasks. Regarding metrics, we evaluate all
models with a widely used incremental metric: the end accuracy on all the seen tasks $A_T=1/T \sum_{i=1}^{i=T}a_{i,T}$
, where $T$ is the total number of tasks and $a_{i,j}$ denotes the accuracy of the $j$-th task once the model has learned the $t$-th task.
We also report the average accuracy $\Tilde{A}_T =\frac{1}{T}\sum_{t=1}^{t=T} A_t$. 

\textbf{Continual PEFT Baselines.}
We evaluate our proposed method (aEPG), against seven continual PEFT baselines. These include: a regularization-based approach—LwF~\citep{li2017learning} implemented with LoRA; prompt-based methods—L2P~\citep{wang2022learning}, Dual Prompt~\citep{wang2022dualprompt}, and CodaPrompt~\citep{smith2023coda}; an ensemble-based method, LAE~\citep{gao2023unified}; and a constraint-based approach, InfLoRA~\citep{liang2024inflora}. Further details on these baselines are provided in the related work section.

\textbf{Entropy Regularization Methods.}
We examine several entropy regularization techniques: \textbf{Focal Loss (FL)}~\citep{lin2017focal}, \textbf{Label Smoothing (LS)}~\citep{meister2020generalized}, and \textbf{Confidence Penalty (CP)}~\citep{DBLP:conf/iclr/PereyraTCKH17}, which encourage higher entropy; and \textbf{Entropy Penalty (EP)}~\citep{grandvalet2004semi,DBLP:conf/iclr/WangSLOD21}, which reduces entropy. (See Appendix~\ref{sec:appendix_loss_function} for detailed loss formulations.) %To adapt these methods for continual PEFT, each loss is applied locally—computed only over the class labels of the current task.

% \textbf{Continual PEFT Baselines.} We compare our approach (adaptive EPG, aEPG) against 7 continual PEFT methods. They include a regularization-based method: LwF~\citep{li2017learning} with LoRA, prompt-based methods: L2P~\citep{wang2022learning}, Dual Prompt~\citep{wang2022dualprompt}, CodaPrompt~\citep{smith2023coda}, an ensemble-based method, LAE~\citep{gao2023unified}, and a constraint-based method, InferLoRA~\citep{liang2024inflora}. Details of baseline methods can be found in the related work section.

% \textbf{Entropy Regularization Methods}. We investigate the following entropy regularization techniques, including: \textbf{Focal loss}~\citep{lin2017focal}, \textbf{Label smoothing}~\citep{meister2020generalized}, and \textbf{Confidence penalty}~\citep{DBLP:conf/iclr/PereyraTCKH17} that increase entropy and \textbf{Entropy penalty}~\citep{grandvalet2004semi,DBLP:conf/iclr/WangSLOD21} that decreases entropy. (see Appendix~\ref{sec:appendix_loss_function} for detailed loss formulations). To adapt these methods for continual PEFT, these losses are applied locally, computed exclusively over the current task's categories.

% \begin{equation}
% \label{eq:local_ce}
%     \mathcal{L} = \frac{1}{|\mathcal{D}_i|} \sum_{(\mathbf{x}, y) \in \mathcal{D}_i} \mathcal{L}_{ce}(\operatorname{mask}(f(\mathbf{x}; \boldsymbol{\theta}, \boldsymbol{\phi})), y),
% \end{equation}

\begin{table}[]
\label{tab:label_noise}
 \caption{Continual learning performance with 20\% symmetric label noise (averaged over 3 runs).}
 \vspace{3pt}
\begin{tabular}{lccc}
\toprule
\textbf{20\% $\eta$} & \textbf{S-Imagenet-R} & \textbf{S-CUB200} & \textbf{CLRS25} \\\midrule
% CE                  & 70.6                & 79.4            & 67.4            \\
% aEPG                & 73.7                & 81.3            & 73.4    \\  
CE	& 70.6 $\pm$ 0.1 &	79.5 $\pm$ 	0.3	& 67.7 $\pm$ 	1.2\\
aEPG &	73.8 $\pm$ 	0.2	& 81.3	$\pm$ 0.0 &	73.4 $\pm$ 	1.9 \\

\bottomrule
\end{tabular}
\end{table}

\textbf{Training Details.} We adopt a ViT-B/16 backbone~\citep{DBLP:conf/iclr/DosovitskiyB0WZ21}, pre-trained on ImageNet-21k and fine-tuned on ImageNet-1k. All experiments use PyTorch with the Adam optimizer (learning rate=0.0005, batch size=256). We initialize classifier heads from $\mathcal{N}(0, 0.001)$, an aspect previously uncontrolled in the release code of previous continual fine-tuning works. Following ~\cite{gao2023unified}, the ViT backbone remains frozen for the first 30 epochs before full fine-tuning for 20 additional epochs (50 total). All PEFT modules (LoRA, Adapters, or Prefix Tuning) are applied to the first 5 transformer blocks with LoRA configured to rank 4. Results for the first 10 blocks show a similar pattern and are omitted. More training details can be found in Appendix~\ref{sec:training_details}.

\subsection{Results}

\textbf{Continual PEFT Results.} Following prior works~\citep{smith2023coda,gao2023unified,liang2024inflora}, we first evaluate performance on ImageNet-R with 5, 10, and 20 task splits. As shown in Table~\ref{tab:sota}, aEPG consistently outperforms seven baseline methods across all settings. In addition, the proposed aEPG loss can be seamlessly combined with stability-enhancing strategies such as fast–slow learning paradigm like LAE~\citep{gao2023unified} or regularization-based approaches e.g. LwF~\citep{li2017learning}. Indeed, Table~\ref{tab:sota} shows that combining aEPG with LAE and LwF yields further gains.  

Unlike DualPrompt or InfLoRA, which depend on specific PEFT architectures, aEPG integrates readily with diverse PEFT frameworks. Table~\ref{tab:lae} confirms that aEPG consistently improves over CE-based optimization when applied to Adapters, LoRA, and Prefix Tuning. We further validate its effectiveness across diverse datasets, including Split-Food101, Split-CUB200, and CLRS.  

\textbf{Noisy Labels.} We also evaluate under 20\% symmetric label noise. As shown in Table~\ref{tab:label_noise}, aEPG substantially outperforms CE in this setting, highlighting its robustness against label noise overfitting.

\textbf{Entropy Dynamics.} We examine the role of entropy during continual fine-tuning of pretrained models. As shown in Table~\ref{tab:focal_all} and Fig.~\ref{fig:focal_imagenet-r}, entropy-increasing methods (label smoothing, confidence penalty, focal loss) consistently degrade performance, whereas entropy-reducing methods (EPG, aEPG, EP) yield clear improvements. This trend also holds under different entropy regularization strengths (see Appendix Fig~\ref{fig:focal_entropy_strength}).

Compared to explicit entropy penalties (EP), aEPG reduces entropy implicitly through sample reweighting: harder examples receive greater importance, lowering prediction entropy. Empirically, aEPG surpasses EP on three of four benchmarks, with the only exception being CUB-200. We attribute this to the low intra-class variance in fine-grained tasks like CUB-200, where the entropy reduction from importance sampling is less effective.

\subsection{Ablation Studies}
\label{sec:ablation_study}
\textbf{Effect of $\alpha$ on Objective Combination.}
Fig.~\ref{fig:alpha} shows the impact of the weighting coefficient $\alpha$ when combining CE and EPG in Eq~\ref{eq:aepg}. We observe: 1) lower $\alpha$ values (emphasizing the RL objective) generally produce superior results, and 2) our adaptive $\alpha$ scheduling strategy consistently outperforms fixed $\alpha$ configurations. % These findings suggest that adaptive entropy annealing leads to a better exploration and exploitation trade-off.

\begin{figure}
    \centering
      \subfigure[ImageNetR]{\includegraphics[width=0.48\linewidth]{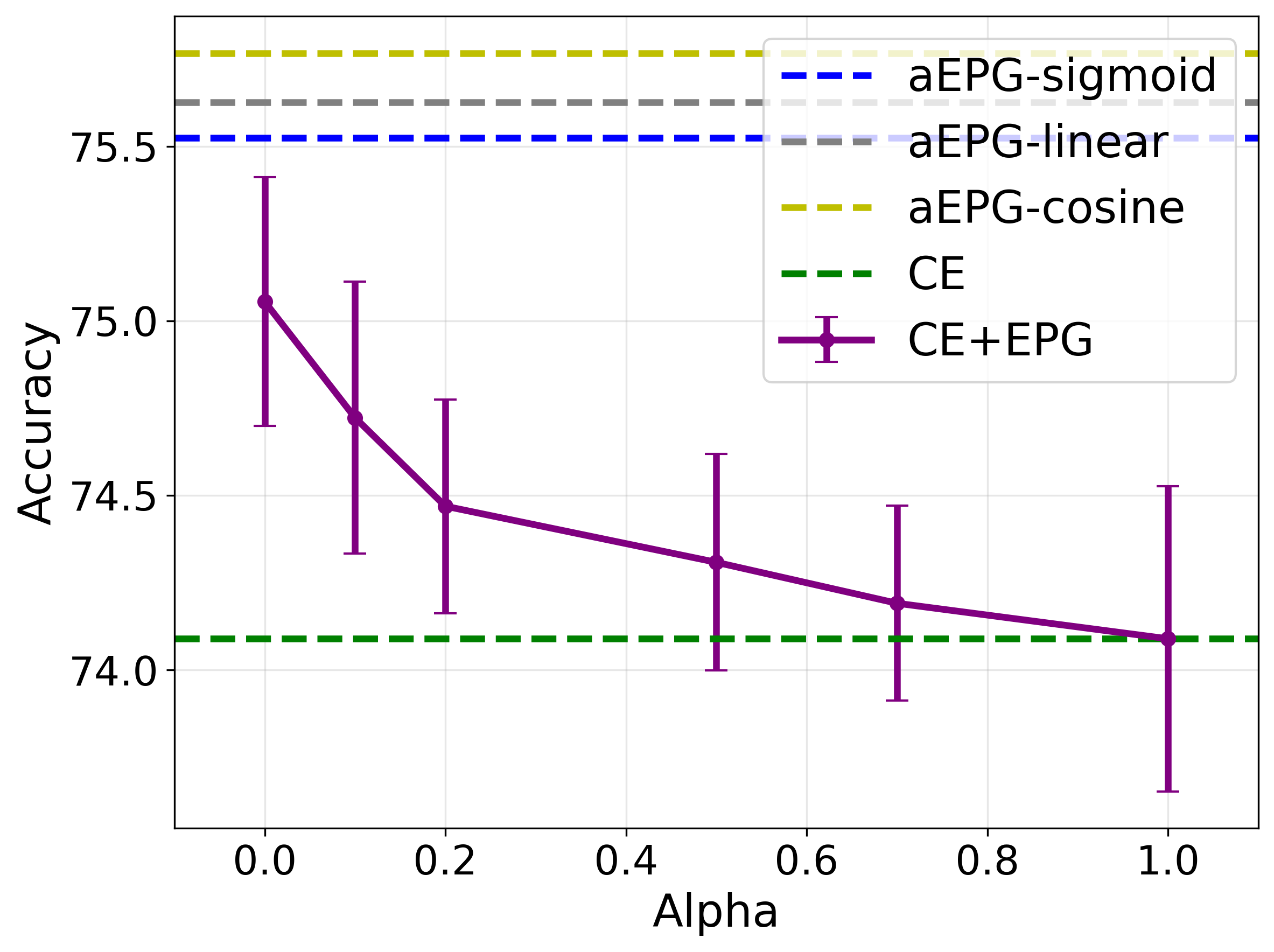}}
% \subfigure[CUB200]{  \includegraphics[width=0.45\linewidth]{fig/alpha_cub200.png}}
      \subfigure[$\alpha$ schedule ]{\includegraphics[width=0.48\linewidth]{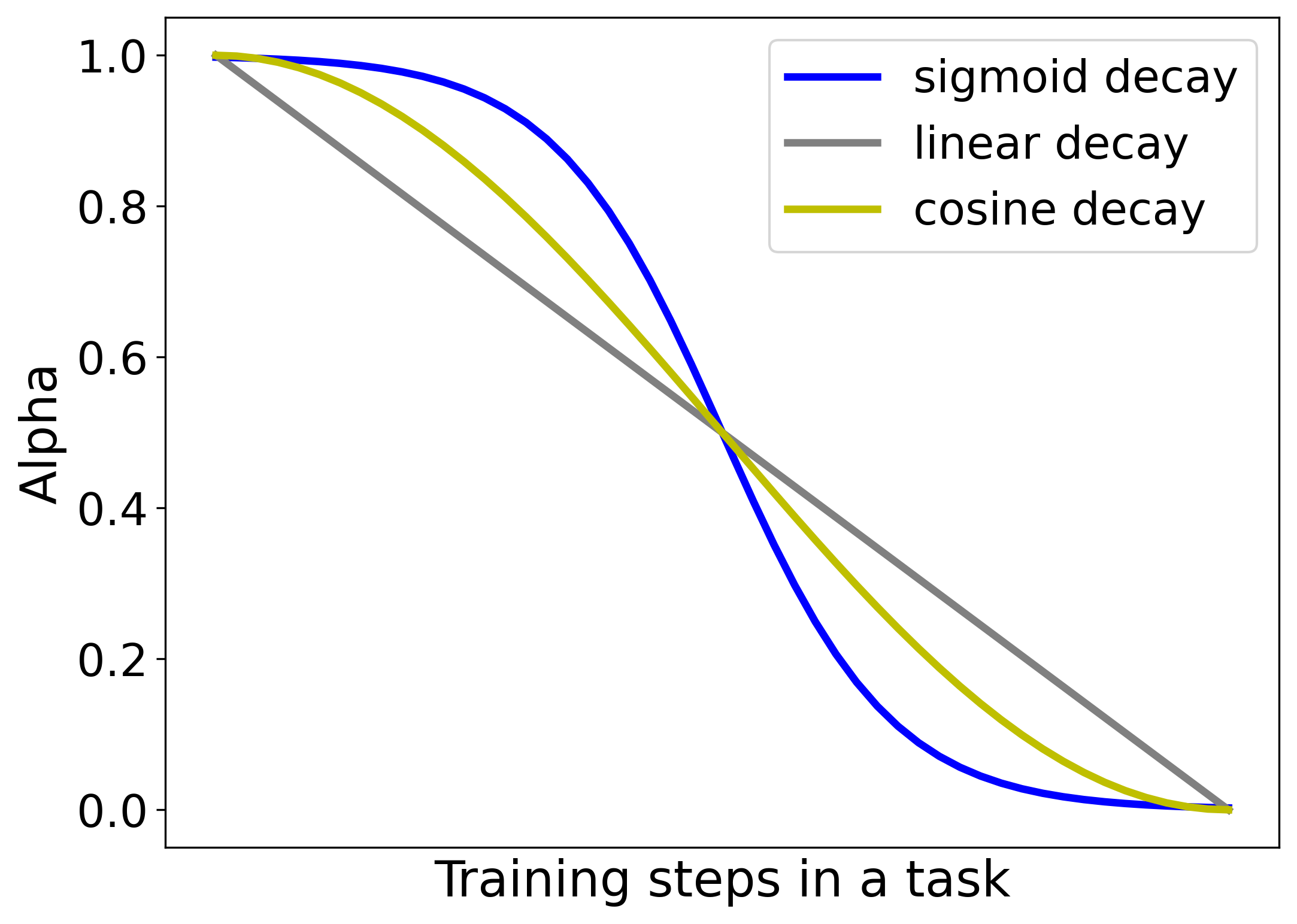}}
    \caption{The effect of $\alpha$ when combining CE and EPG with $\alpha \mathcal{L}_{CE} + (1-\alpha) \mathcal{L}_{EPG}$. Results show the mean and standard deviation across 5  runs.}
    \label{fig:alpha}
\end{figure}

\textbf{Entropy Annealing Mechanism.} We compare entropy annealing schedules, including linear decay ($\alpha_t = \tfrac{T-t}{T}$) and cosine decay ($\alpha_t = \tfrac{1}{2} + \tfrac{1}{2}\cos(\pi \tfrac{t}{T})$). As shown in Fig.~\ref{fig:alpha}, both perform comparably to the default sigmoid schedule. For the sigmoid, $\tau=6$ is used by default, while $\tau=4$ and $\tau=8$ yield similarly smooth transitions from 1 to 0 with no notable performance drop (see Appendix Table~\ref{tab:tau}).

\textbf{Training from Scratch.} Although our focus is continual fine-tuning, the observation that EPG can optimize the 0-1 loss also applies to standard supervised learning. To test this, we trained ResNet-50 from scratch on CIFAR-10 and CIFAR-100 for 350 epochs using the customized learning rate schedule from prior work~\citep{DBLP:conf/iclr/PereyraTCKH17} (see Appendix~\ref{sec:training_details}). Table~\ref{tab:train-from-scratch} shows that, unlike in pretrained models, pure 0-1 optimization ($\alpha=0$) fails to converge on CIFAR-100, indicating that randomly initialized networks require stronger exploration than pretrained weights. Interestingly, combining EPG with CE (e.g., $\alpha=0.2$ or $0.5$) outperforms CE alone, yielding gains of about 2\% on CIFAR-100 and 0.5\% on CIFAR-10. These findings suggest that incorporating 0-1 loss into CE not only improves continual fine-tuning but also enhances training from scratch. Additional results on the evolution of training accuracy and entropy in this setting are provided in Appendix Fig.~\ref{fig:train_scratch}.

\begin{table}[t]
\centering
\caption{ ResNet-50 training from scratch. Test accuracy averaged over 3 independent runs}
\vspace{4pt}
\begin{tabular}{lcc}
\toprule
Method & CIFAR100 & CIFAR10 \\
\midrule
$\alpha=0$ & $7.32 \pm 0.70$ & $92.22 \pm 0.51$ \\
$\alpha=0.2$ & \textbf{80.80} $\pm$ 0.13 & \textbf{95.81} $\pm$ 0.05 \\
$\alpha=0.5$ & $79.95 \pm 0.20$ & $95.62 \pm 0.08$ \\
$\alpha=1$ (CE) & $77.95 \pm 0.78$ & $95.31 \pm 0.18$ \\
\bottomrule
\end{tabular}
\label{tab:train-from-scratch}
\end{table}

\section{Conclusion }
In this work, we revisit the conventional use of cross-entropy loss for learning new data in continual learning through the lens of reinforcement learning and entropy. By contrasting CE with 0-1 loss optimization—solved via expected policy gradient (EPG)—our theoretical and empirical analyses show that CE implicitly emphasizes low-confidence predictions, resulting in excessive exploration and high-entropy output distributions. To better balance exploration and exploitation, we propose an adaptive entropy annealing strategy (aEPG) that gradually interpolates between CE and EPG, yielding consistent improvements across diverse benchmarks and parameter-efficient fine-tuning (PEFT) architectures. Moreover, we question the prevailing assumption that high-entropy regularization aids classification, demonstrating instead that lower-entropy predictions consistently enhance class-incremental learning with pretrained vision transformers.
%Our findings suggest that excessive exploration can disrupt the pretrained model’s knowledge, advocating for a more exploitative learning paradigm in CL.

%\textbf{Limitations}. While our method demonstrates strong performance in continual fine-tuning with vision transformers, it has several limitations. First, our theoretical and empirical analyses assume a standard supervised setting with clean, hard labels, leaving EPG’s robustness to noisy or ambiguous samples an open question for future work. Second, our experiments focus exclusively on class-incremental learning with pretrained vision transformers, and further validation is needed to assess generalizability to other architectures (e.g., CNNs) or modalities (e.g., language models).

\bibliography{CL}
\bibliographystyle{apalike}
\section*{Checklist}

% % %%% BEGIN INSTRUCTIONS %%%
% The checklist follows the references. For each question, choose your answer from the three possible options: Yes, No, Not Applicable.  You are encouraged to include a justification to your answer, either by referencing the appropriate section of your paper or providing a brief inline description (1-2 sentences). 
% Please do not modify the questions.  Note that the Checklist section does not count towards the page limit. Not including the checklist in the first submission won't result in desk rejection, although in such case we will ask you to upload it during the author response period and include it in camera ready (if accepted).

% \textbf{In your paper, please delete this instructions block and only keep the Checklist section heading above along with the questions/answers below.}
% % %%% END INSTRUCTIONS %%%

\begin{enumerate}

  \item For all models and algorithms presented, check if you include:
  \begin{enumerate}
    \item A clear description of the mathematical setting, assumptions, algorithm, and/or model. [Yes]
    Section 3.1 and Section 4 describe the problem settings and models used. Section 3.4 describes the algorithm.
    \item An analysis of the properties and complexity (time, space, sample size) of any algorithm. [No]
    The proposed method does not change complexity compared to the baseline, i.e. cross entropy loss.
    \item (Optional) Anonymized source code, with specification of all dependencies, including external libraries. [Yes]
  \end{enumerate}

  \item For any theoretical claim, check if you include:
  \begin{enumerate}
    \item Statements of the full set of assumptions of all theoretical results. [Yes] See Proposition 1-3.
    \item Complete proofs of all theoretical results. [Yes] See Proposition 1-3.
    \item Clear explanations of any assumptions. [Yes]  See Proposition 1-3    
  \end{enumerate}

  \item For all figures and tables that present empirical results, check if you include:
  \begin{enumerate}
    \item The code, data, and instructions needed to reproduce the main experimental results (either in the supplemental material or as a URL). [Yes] We provide the code and instructions in the supplemental material.
    \item All the training details (e.g., data splits, hyperparameters, how they were chosen). [Yes]. See Appendix~\ref{sec:training_details}.
    \item A clear definition of the specific measure or statistics and error bars (e.g., with respect to the random seed after running experiments multiple times). [Yes] This information is described in the table and figure captions.
    \item A description of the computing infrastructure used. (e.g., type of GPUs, internal cluster, or cloud provider). [Yes]. See Appendix~\ref{sec:training_details}
  \end{enumerate}

  \item If you are using existing assets (e.g., code, data, models) or curating/releasing new assets, check if you include:
  \begin{enumerate}
    \item Citations of the creator If your work uses existing assets. [Yes] Data and Model is cited in the Experiment Section. The code is cited in Appendix~\ref{sec:training_details}.
    \item The license information of the assets, if applicable. [Yes]
    \item New assets either in the supplemental material or as a URL, if applicable. [Yes] We provide code in the supplemental material.
    \item Information about consent from data providers/curators. [Yes]
    \item Discussion of sensible content if applicable, e.g., personally identifiable information or offensive content. [Not Applicable]
  \end{enumerate}

  \item If you used crowdsourcing or conducted research with human subjects, check if you include:
  \begin{enumerate}
    \item The full text of instructions given to participants and screenshots. [Not Applicable]
    \item Descriptions of potential participant risks, with links to Institutional Review Board (IRB) approvals if applicable. [Not Applicable]
    \item The estimated hourly wage paid to participants and the total amount spent on participant compensation. [Not Applicable]
  \end{enumerate}

\end{enumerate}

\clearpage
\appendix
\thispagestyle{empty}

% Supplementary material: To improve readability, you must use a single-column format for the supplementary material.
\onecolumn

\section{Proofs}
\subsection{Proof of Proposition~\ref{proposition_01}}
\label{sec:appendix_proof_01}
\begin{proof}
By interpreting $h_\theta(x)$ as the policy $\pi_\theta(a|y)$ in Eq~\ref{eq:rl_objective} and applying a constant baseline of value 1 to the reward function $\mathcal{R}_{x,a}$ (Eq.~\ref{eq_reward_def}), we obtain:
\begin{equation}
\begin{aligned}
       J_h(\theta) &=\mathbb{E}_{x\sim d(x),a\sim h_\theta}[r]\\
      & = 1-\sum_{x \in \mathcal{X}} d(x) \sum_{a \in \mathcal{A}} h_\theta(a|x) (-\mathcal{R}_{x,a}+1)\\
&= 1 - \mathcal{L}_{01}(h_\theta) .
\end{aligned}
\end{equation}
The constant offset does not affect the optimization objective, thus establishing the equivalence.
\end{proof}
\subsection{Proof of Proposition~\ref{proposition:label noise}}
\label{sec:appendix_proof_label_noise}

\begin{proof} 
Under symmetric label noise with rate $\eta$, the objective of EPG becomes
\begin{equation}
 \begin{aligned}
    J^\eta(\theta)
    &= (1-\eta) \sum_{x \in \mathcal{X}} d(x)\Bigg(\sum_{a \neq y} \pi_\theta(a|x) \cdot 0 + \pi_\theta(y|x) \cdot 1\Bigg) \\
    &\quad + \sum_{x \in \mathcal{X}} d(x)\frac{\eta}{K-1}\sum_{a \neq y} \pi_\theta(a|x) \cdot 1 \\
    &= (1-\eta) J(\theta) + \sum_{x \in \mathcal{X}} d(x)\frac{\eta}{K-1}\big(1-\pi_\theta(y|x)\big) \\
    &= \Big(1-\tfrac{K}{K-1}\eta\Big) J(\theta) + \tfrac{\eta}{K-1}.
 \end{aligned}
\end{equation}

Let $\theta^\eta_*$ denote the global maximizer of $J^\eta(\theta)$ under noisy labels, i.e., $J^\eta(\theta^\eta_*) - J^\eta(\theta) \geq 0$ for all $\theta$. Then, we have
\begin{equation}
\begin{aligned}
    J(\theta^\eta_*) - J(\theta) 
    &= \frac{1}{1 - \tfrac{K}{K-1}\eta}\big(J^\eta(\theta^*) - J^\eta(\theta)\big) \\
    &\geq 0, \quad \text{whenever } \eta < 1 - \tfrac{1}{K}.
\end{aligned}
\end{equation}
Therefore, the optimizer $\theta^\eta_*$ under noisy labels remains the global maximizer of $J(\theta)$ under clean labels, proving noise tolerance of EPG under the stated condition.
\end{proof}
\subsection{Proof of Proposition~\ref{proposition:KL}}
\label{sec:appendix_proof_kl}
\begin{proof}
Let $p_\theta(y|x)$ denote the model's predictive distribution, and let $q_\eta(y|x)$ denote the label-noise target distribution:
\[
q_\eta(y_k|x) =
\begin{cases} 
1 - \eta & \text{if } y_k = y^*, \\[1mm]
\frac{\eta}{K-1} & \text{otherwise},
\end{cases}
\]
where $y^*$ is the observed label, $K$ is the number of classes, and $\eta \in (0,1)$ is the label noise rate.

The KL divergence between $p_\theta$ and $q_\eta$ is
\[
D_\text{KL}(p_\theta \parallel q_\eta) 
= \sum_{k=1}^K p_\theta(y_k|x) \log \frac{p_\theta(y_k|x)}{q_\eta(y_k|x)}
= -H(p_\theta) - \sum_{k=1}^K p_\theta(y_k|x) \log q_\eta(y_k|x),
\]
where $H(p_\theta) = -\sum_k p_\theta(y_k|x) \log p_\theta(y_k|x)$ is the entropy of $p_\theta$.

Consider the second term, we have:
\begin{equation}
    \begin{aligned}
   - D_\text{KL}(p_\theta \parallel q_\eta) - H(p_\theta)
&= \sum_{k=1}^K p_\theta(y_k|x) \log q_\eta(y_k|x) \\
       &= p_\theta(y^*|x) \log(1-\eta) + \sum_{y_k \neq y^*} p_\theta(y_k|x) \log \frac{\eta}{K-1} \\
       & =  p_\theta(y^*|x) \log(1-\eta) +  \log \frac{\eta}{K-1} \sum_{y_k \neq y^*} p_\theta(y_k|x)\\
       & =p_\theta(y^*|x) \log(1-\eta) + \big(1 - p_\theta(y^*|x)\big) \log \frac{\eta}{K-1}\\
       &= \log(1-\eta) \, \mathbb{E}_{p_\theta}[\mathcal{R}(y,x)] + \log\Big(\frac{\eta}{K-1}\Big) \big(1 - \mathbb{E}_{p_\theta}[\mathcal{R}(y,x)]\big)\\
       &= A \cdot \mathbb{E}_{p_\theta}[\mathcal{R}(y,x)]+B
    \end{aligned}
\end{equation}
where $A=\log(1-\eta) - \log\Big(\frac{\eta}{K-1}\Big)$ and $B=\log\Big(\frac{\eta}{K-1}\Big) $

Therefore, minimizing
$
D_\text{KL}(p_\theta \parallel q_\eta) + H(p_\theta) $ is equal to maximizing  $\mathbb{E}_{p_\theta}[\mathcal{R}(y,x)]$, with $\eta \in (0,1)$

\end{proof}

\section{Loss function details}
\label{sec:appendix_loss_function}
Table~\ref{tab:loss} compares the objective functions of different entropy regularization methods. Focal loss, label smoothing, and confidence penalty increase entropy, whereas EPG, aEPG, and entropy penalty reduce it.
\begin{table*}[htb]
%\label{tab:loss}
\centering
\caption{Loss functions and their effects on entropy}
\label{tab:loss}
\begin{tabular}{l l l}
\toprule
\textbf{Training Method} & \textbf{Loss Function} & \textbf{Entropy} \\
\midrule
Cross Entropy & 
$L_{CE} = -\sum q \log p_{\theta}$ &
 Baseline \\
\midrule
Confidence Penalty& 
$L_{CP} = L_{CE} - \beta H(p_{\theta} )$ &
 \\
Label Smoothing & 
%$L_{LS}(\theta) = L_{CE}(\theta) + \beta D_{KL}(u \| p_{\theta})$ 
$L_{LS}=(1-\gamma) L_{CE}+\gamma D_{\mathrm{KL}}\left(u \| p_\theta\right)$
&
Increase entropy $\uparrow$ \\
Focal loss & $L_{FL}=(1-p_\theta)^\gamma L_{CE}$ \\ \midrule
Expected Policy Gradient & $L_{EPG}= -\mathbb{E}_{p_\theta}[q]$\\
Entropy penalty & $L_{EP} = L_{CE} + H(p_{\theta})$ & Decrease entropy $\downarrow$\\
aEPG & $L_{aEPG}=\alpha_t L_{CE}+(1-\alpha_t) L_{EPG}$\\
% Generalized Entropy Regularization, $D_{J\alpha}$ & 
% $L_{GEN}(\theta) = L_{CE}(\theta) + \beta D_{J\alpha}(u \| p_{\theta})$ &
% -- \\
\bottomrule
\end{tabular}
\end{table*}
\section{Implementation Details}
\label{sec:training_details}
\textbf{Continual Fine-tuning Experiments}. We evaluate all methods using consistent pretraining weights for vision transformer (vit\_base\_patch16\_224 from timm library)\footnote{\url{https://storage.googleapis.com/vit_models/augreg/B_16-i21k-300ep-lr_0.001-aug_medium1-wd_0.1-do_0.0-sd_0.0--imagenet2012-steps_20k-lr_0.01-res_224.npz}} and optimization settings. 
The detailed hyperparameter settings for all algorithms are listed in Table~\ref{tab:hyperparameter_continual}. 
For L2P, DualPrompt, CodaPrompt, InfLoRA, and LAE, we adopt the key algorithm-specific hyperparameters following their original papers and official implementations. Our implementation mostly builds upon the LAE codebase~\footnote{\url{https://github.com/gqk/LAE}}, which are used to implement aEPG, LwfLoRA, and LAE. For DualPrompt, we use the PyTorch implementation from \cite{wang2022dualprompt}~\footnote{\url{https://github.com/JH-LEE-KR/dualprompt-pytorch}}. For L2P and CodaPrompt, we use the official PyTorch implementation from ~\cite{smith2023coda}~\footnote{\url{https://github.com/GT-RIPL/CODA-Prompt}}. InfLoRA results are based on the code released at ~\cite{liang2024inflora}~\footnote{\url{https://github.com/liangyanshuo/InfLoRA}}.

Our experiments were conducted on internal clusters consisting of NVIDIA GPUs. The average runtime for a single dataset in one independent run ranges between 1--5 hours, depending on the task complexity.
% ViT version.    We use vit\_base\_patch16\_224 from timm library.

\begin{table*}[]
    \centering
      \caption{Hyperparameter setting in the continual fine-tuning experiments.}
    \begin{tabular}{c|c} \toprule
    Hyperparameters & Settings \\ \midrule
    Pretrained model & vit\_base\_patch16\_224\\
  Training epoch & 50 \\
  Backbone freeze epoch & 30 \\
        Batch size & 256\\
        Learning rate & 0.0005  \\
        Optimizer & Adam ($\beta1=0.9, \beta2=0.999$, eps=1e-08)\\ 
        Weight decay & 0 \\
        Gradient clipping & None \\
        Classifier initialization & Normal distribution with std of 0.001 \\
        \multirow{2}{*}{Augmentation} & Random Resized Crop: scale = (0.05, 1.0), ratio = (3. / 4., 4. / 3.), \\
        & Random Horizontal Flip (p=0.5)\\
        \midrule
        Focal loss & gamma: 0.5,1,2 \\
        Label smoothing & smooth parameter: 0.01, 0.05, 0.1 \\
        Confidence penalty & penalty intensity: 0.1,0.2\\
        aEPG & tau = 6 \\
        EP & beta = 1 \\
        LwF & KL coef: 1, KL temperature: 1 \\
        Dualprompt & $L_g=5$, $L_e=20$ \\
        InfLoRA & $\epsilon=1e-8$, lamb=0.99, lame=1.0, rank=5 \\
        LAE  & EMA decay: 0.999\\
        \midrule
        LoRA & block: [0-4], rank = 4\\
        Adapter& block: [0-4], down\_sample = 5\\
        Prefix &block: [0-4], length = 10\\
        \bottomrule
    \end{tabular}
  
    \label{tab:hyperparameter_continual}
\end{table*}

\textbf{Train from Sratch}. All models were trained for 350 epochs with a learning rate reduced by a factor of 10 at epochs 150 and 225. We used Stochastic Gradient Descent (SGD) with a batch size of 256 and momentum of 0.9. We report mean performance metrics with standard deviations across 3 independent runs with different random seeds. Detail hyperparameters are shown in Table~\ref{tab:hyper_scratch}.

\begin{table}[]
    \centering
      \caption{Train from scratch hyperparameter setting}
    \begin{tabular}{c|c} \toprule 
    Hyperparameters & Setting \\ \midrule
          Batch size & 256\\
  Training epoch & 350 \\
LR milestone & 100,225 \\
        Learning rate & 0.1,0.01,0.001  \\
        Optimizer & SGD\\
        Momentum & 0.9 \\
        Weight decay & 0 \\
        Gradient clipping & None \\
        \multirow{2}{*}{Augmentation} & Random Crop: padding=4, \\
        & Random Horizontal Flip (p=0.5)\\
        \bottomrule
    \end{tabular}
  
    \label{tab:hyper_scratch}
\end{table}

    % "vit_base_patch16_224": _cfg(
    %     url="https://storage.googleapis.com/vit_models/augreg/"
    %     "B_16-i21k-300ep-lr_0.001-aug_medium1-wd_0.1-do_0.0-sd_0.0--imagenet2012-steps_20k-lr_0.01-res_224.npz"
 
%%%%%%%%%%%%%%%%%%%%%%%%%%%%%%%%%%%%%%%%%%%%%%%%%%%%%%%%%%%%
\section{Additional Experiment Results}
\subsection{Training from Scratch}
\textbf{Evolution of train accuracy and entropy}. Figure~\ref{fig:train_scratch} illustrates the test accuracy and entropy evolution during training on CIFAR100 and CIFAR10 from random initialization. Consistent with our finding in the continual fine-tuning experiments, EPG optimization yields faster entropy convergence than CE optimization, as shown in Fig.~\ref{fig:train_scratch}c in the appendix. However, when training from a randomly initialized model, we observe that EPG's entropy decreases excessively, ultimately hindering the learning process, a phenomenon not observed when initializing from pretrained models. We observe that smaller alpha values accelerate entropy convergence, with $\alpha=0.2$ achieving optimal performance (2\% improvement over standard cross-entropy). This demonstrates the advantage of combining 0-1 loss with cross-entropy. Notably, pure 0-1 optimization ($\alpha=0$) fails to converge effectively for CIFAR100, unlike in pretrained models. This suggests that randomly initialized networks require stronger initial exploration.
\begin{figure}
    \centering
   % \subfigure[The effect of alpha]{ \includegraphics[width=0.32\linewidth]{fig/alpha_scratch_cifar100_cifar100.png}}
   % \subfigure[Training accuracy]{  \includegraphics[width=0.3\linewidth]{fig/cifar100_train_acc_bold.png}}
\subfigure[Test accuracy]{
          \includegraphics[width=0.4\linewidth]{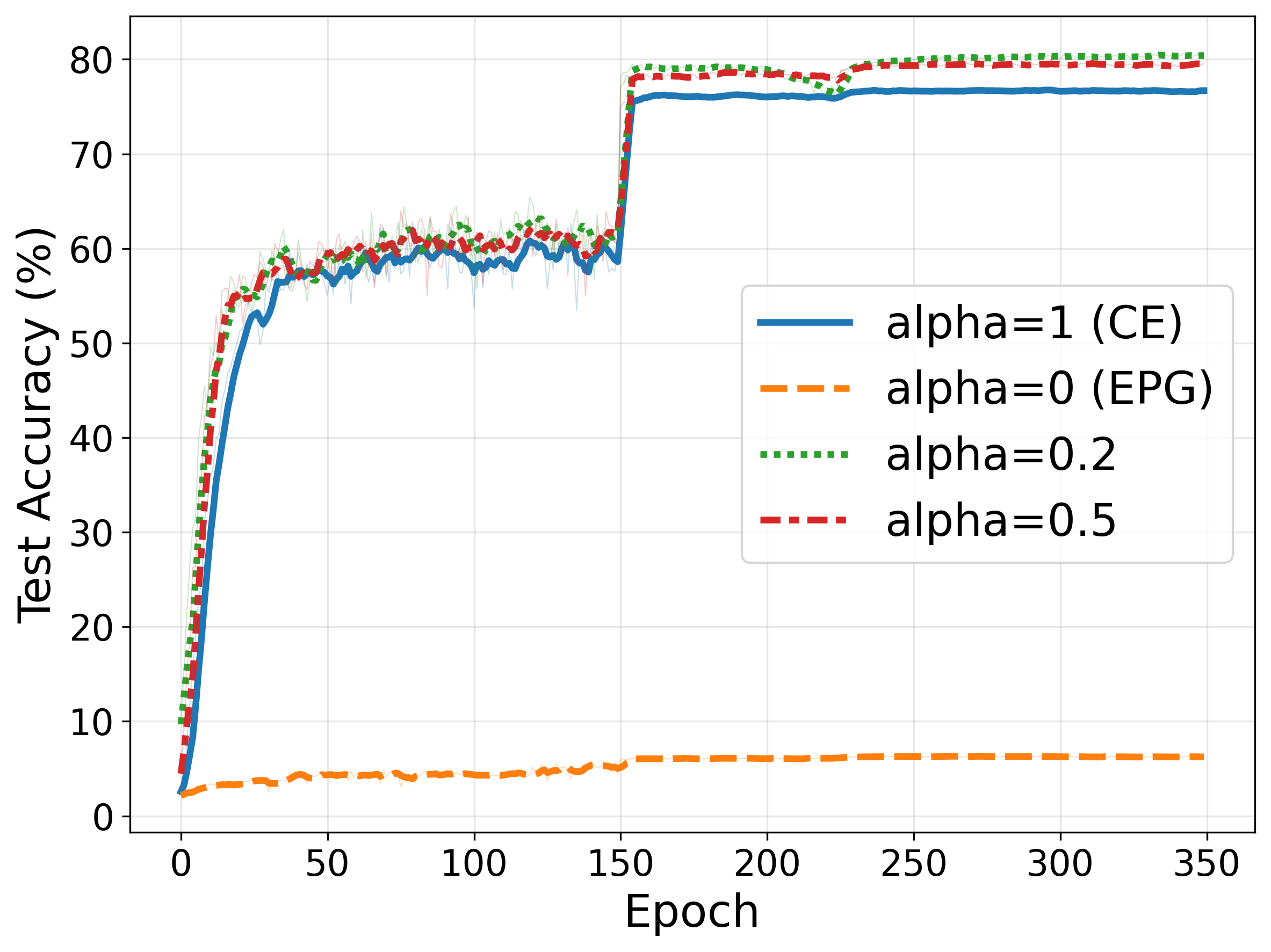}}
          \subfigure[Entropy]{
          \includegraphics[width=0.4\linewidth]{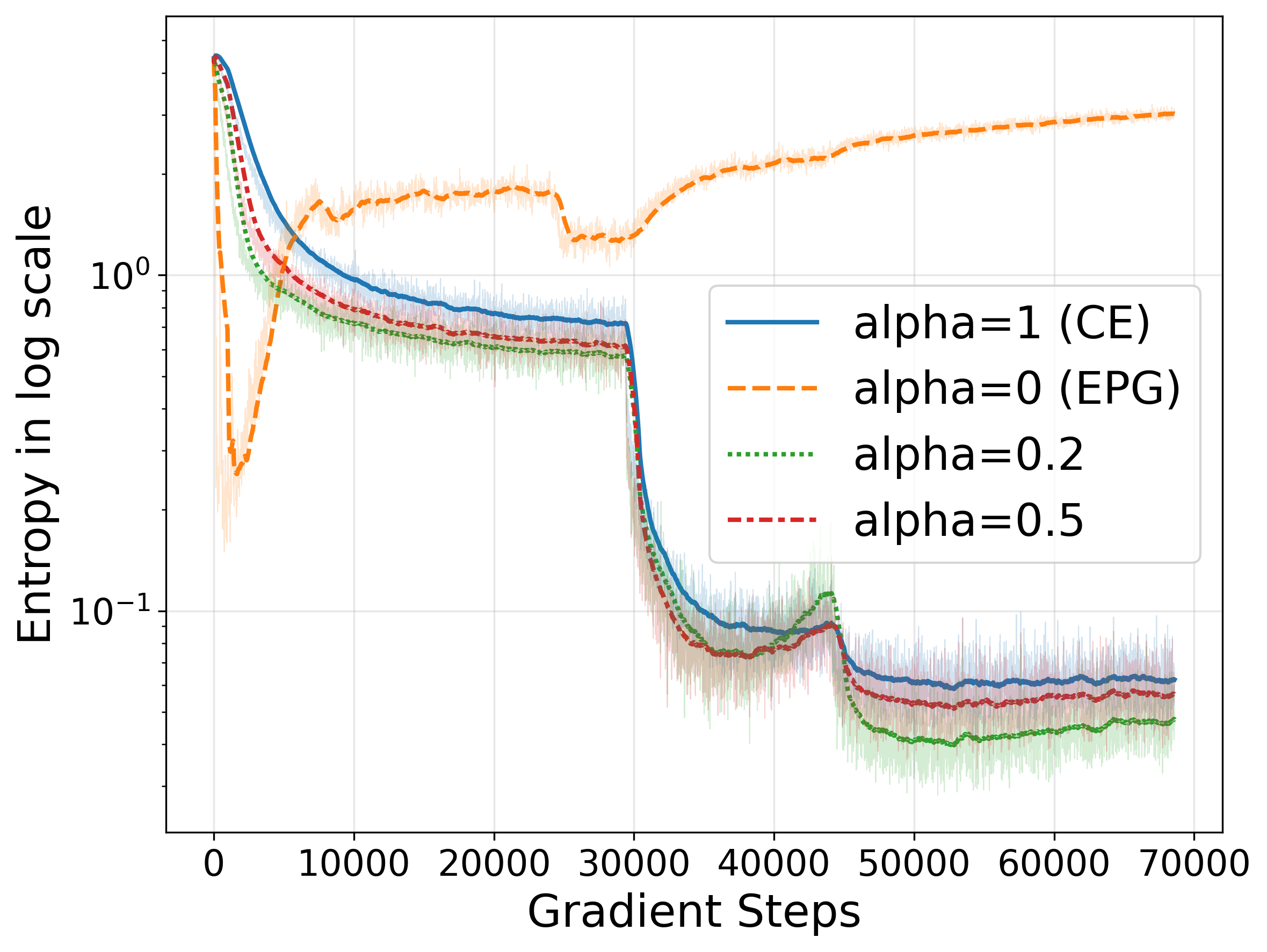}}
    \caption{Training CIFAR100 with ResNet50 from scratch. CE-EPG with an alpha value of 0.2 outperforms the standard CE loss (Best test accuracy: 81\% vs. 78\%.)}
    \label{fig:train_scratch}
\end{figure}
\subsection{Continual Fine-tuning Results}
\label{sec:appendix_results}
\textbf{Entropy Dynamics in Food101}. Figure~\ref{fig:focal_food101} illustrates the entropy dynamics on the Split-Food101 dataset, revealing trends similar to those observed on Split-ImageNetR. Compared to cross-entropy loss, Expected Policy Gradient (EPG), adaptive EPG (aEPG), and Entropy Penalty (EP) achieve lower entropy and higher accuracy. In contrast, focal loss, label smoothing, and confidence penalty (CP) result in higher entropy and degraded performance. Label smoothing is particularly detrimental: even with a small smoothing parameter (0.01), it reduces final accuracy by approximately 5\%.

\textbf{Strength of Entropy Regularization} Figure~\ref{fig:focal_entropy_strength} analyzes the effect of entropy regularization strength in focal loss, label smoothing, and confidence penalty. Increasing regularization typically leads to substantially higher entropy, which in turn degrades performance. For example, label smoothing with a parameter of 0.1 performs worse than with 0.01, reinforcing our observation that excessive exploration harms continual fine-tuning. The only exception is focal loss: both gamma=2 and gamma=0.5 underperform compared to gamma=1.
\begin{figure}
    \centering
\subfigure[Incremental test accuracy]{\includegraphics[width=0.45\linewidth]{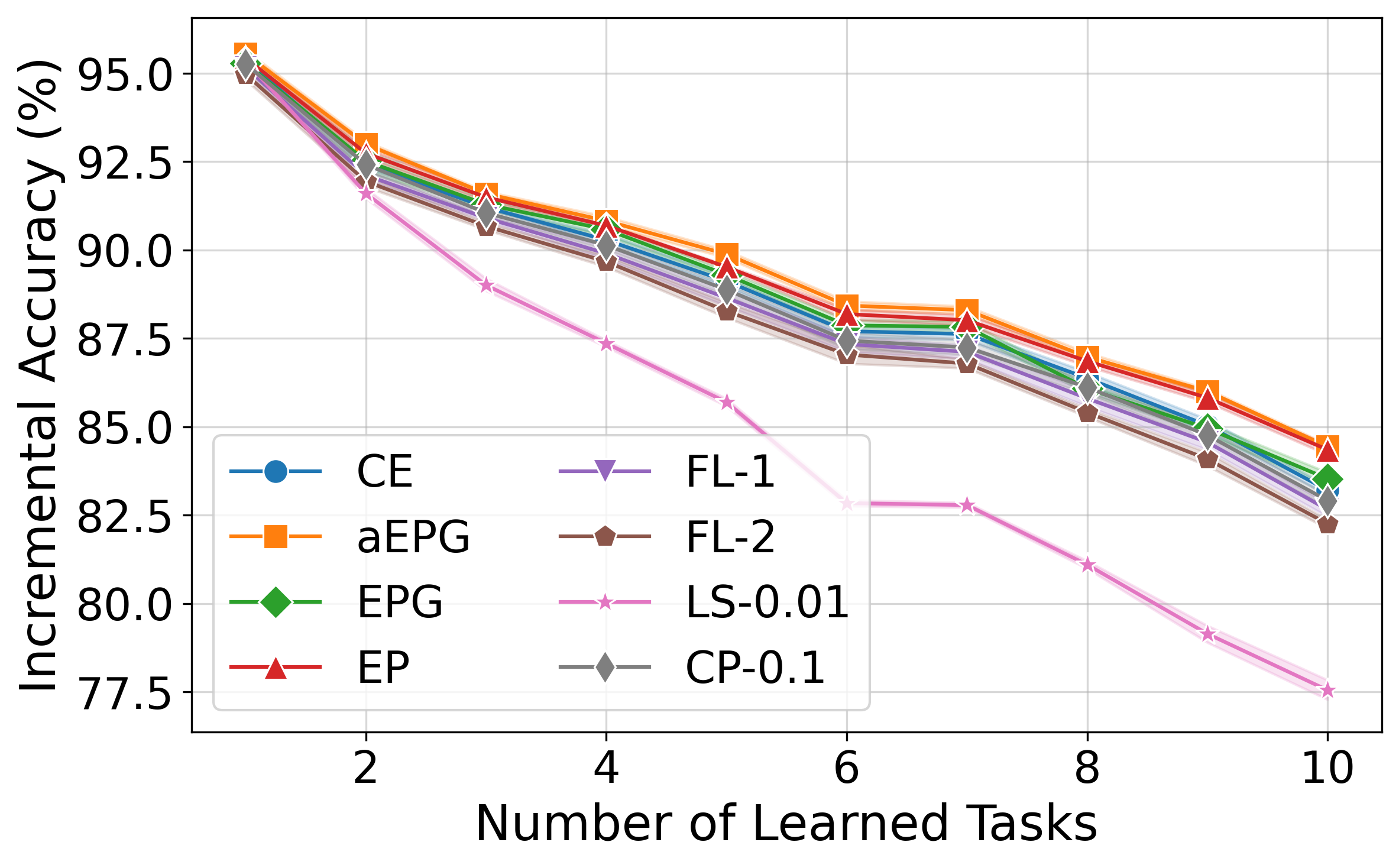}}
\subfigure[Entropy of output distribution]{\includegraphics[width=0.53\linewidth]{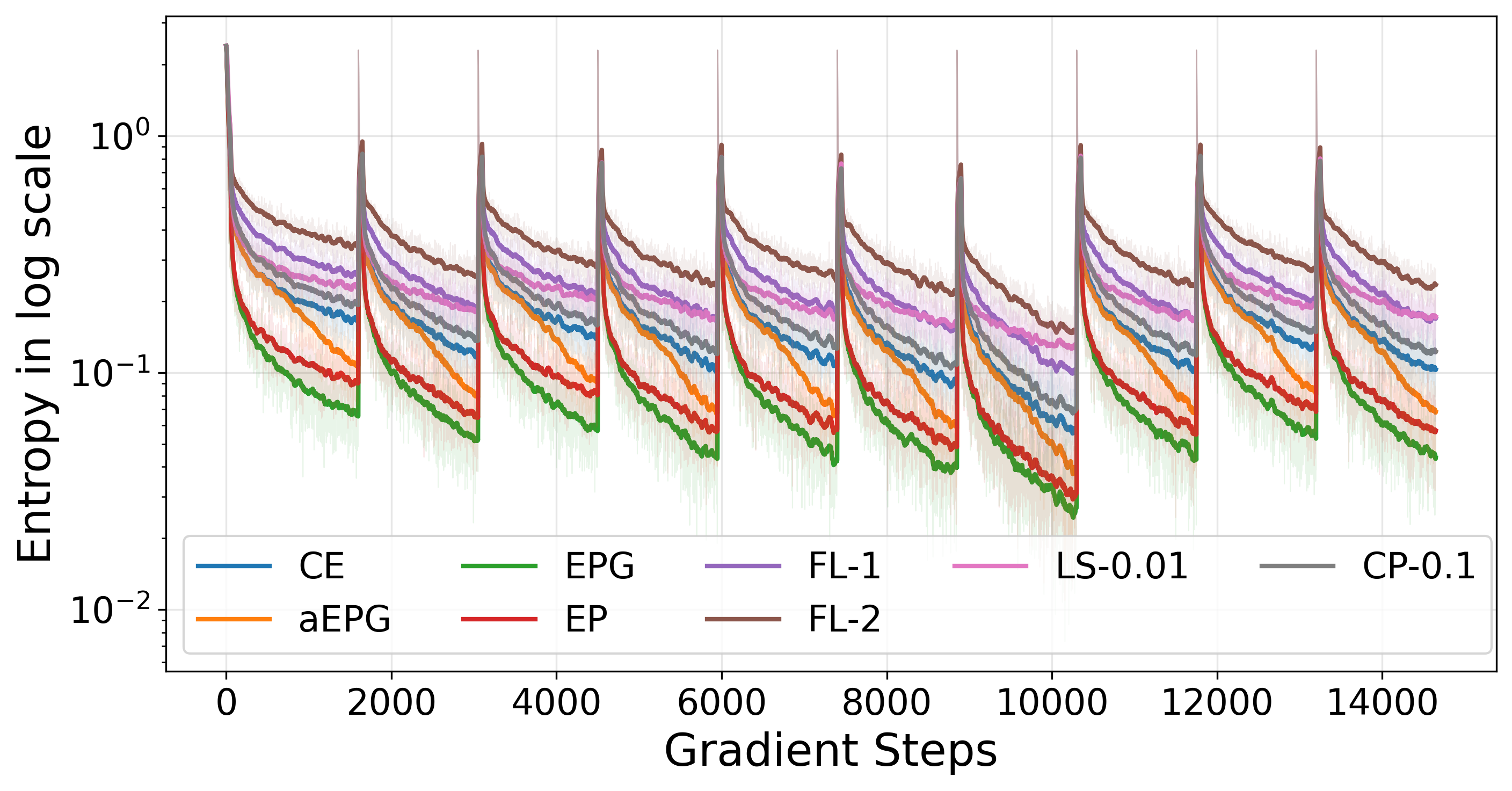}}
    \caption{Entropy dynamics in continual fine-tuning VisionTransformers on Split-Food101. Compared to the cross-entropy loss, Expected Policy Gradient (EPG), adaptive EPG (aEPG), and Entropy Penalty (EP) lead to lower entropy and improved accuracy. In contrast, focal loss, label smoothing, and confidence penalty (CP) lead to higher entropy and worse performance. }
    \label{fig:focal_food101}
\end{figure}

\begin{figure}
    \centering
\subfigure[Focal loss]{\includegraphics[width=0.44\linewidth]{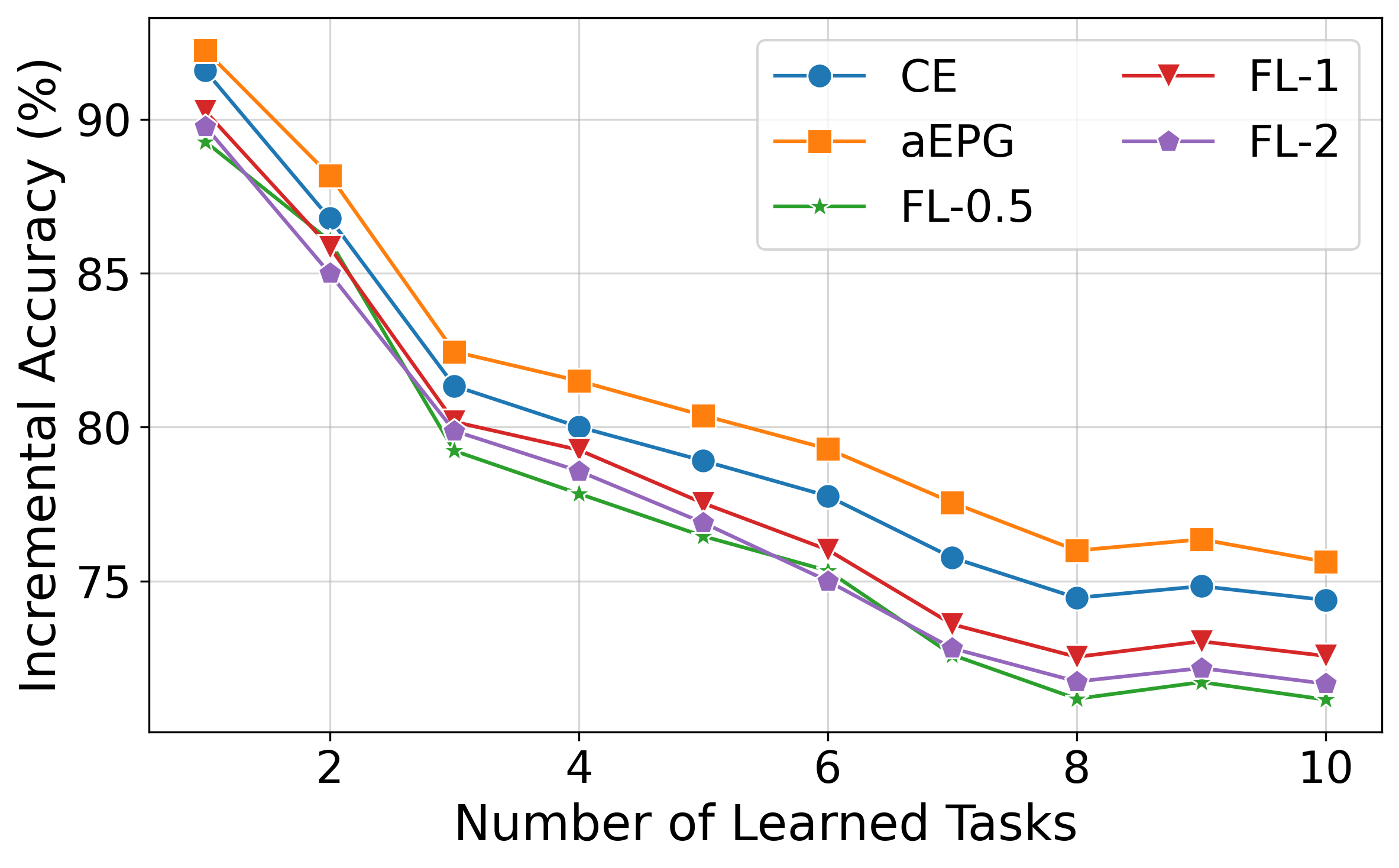}
\includegraphics[width=0.55\linewidth]{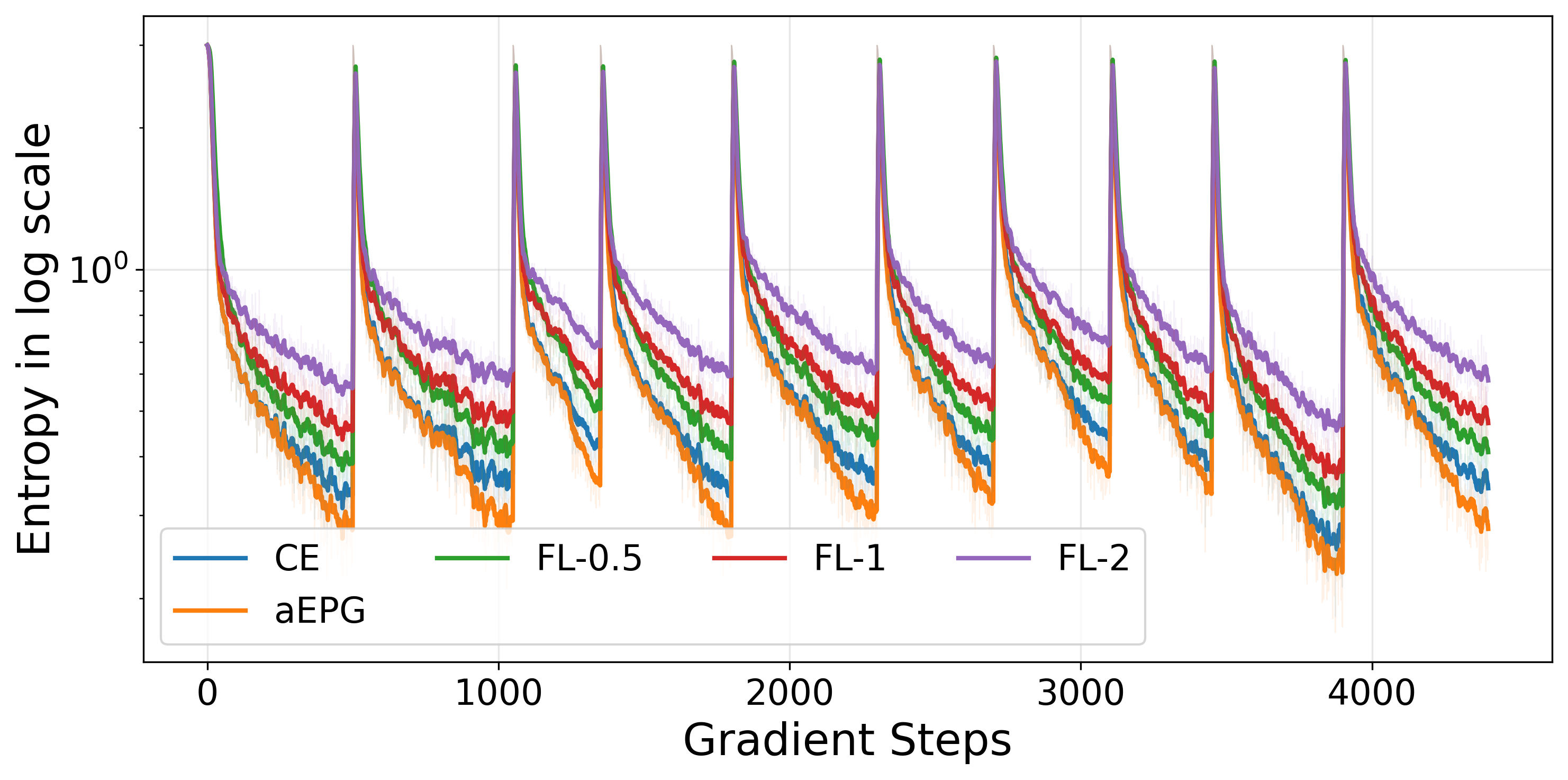}}
\subfigure[Label smoothing]{\includegraphics[width=0.44\linewidth]{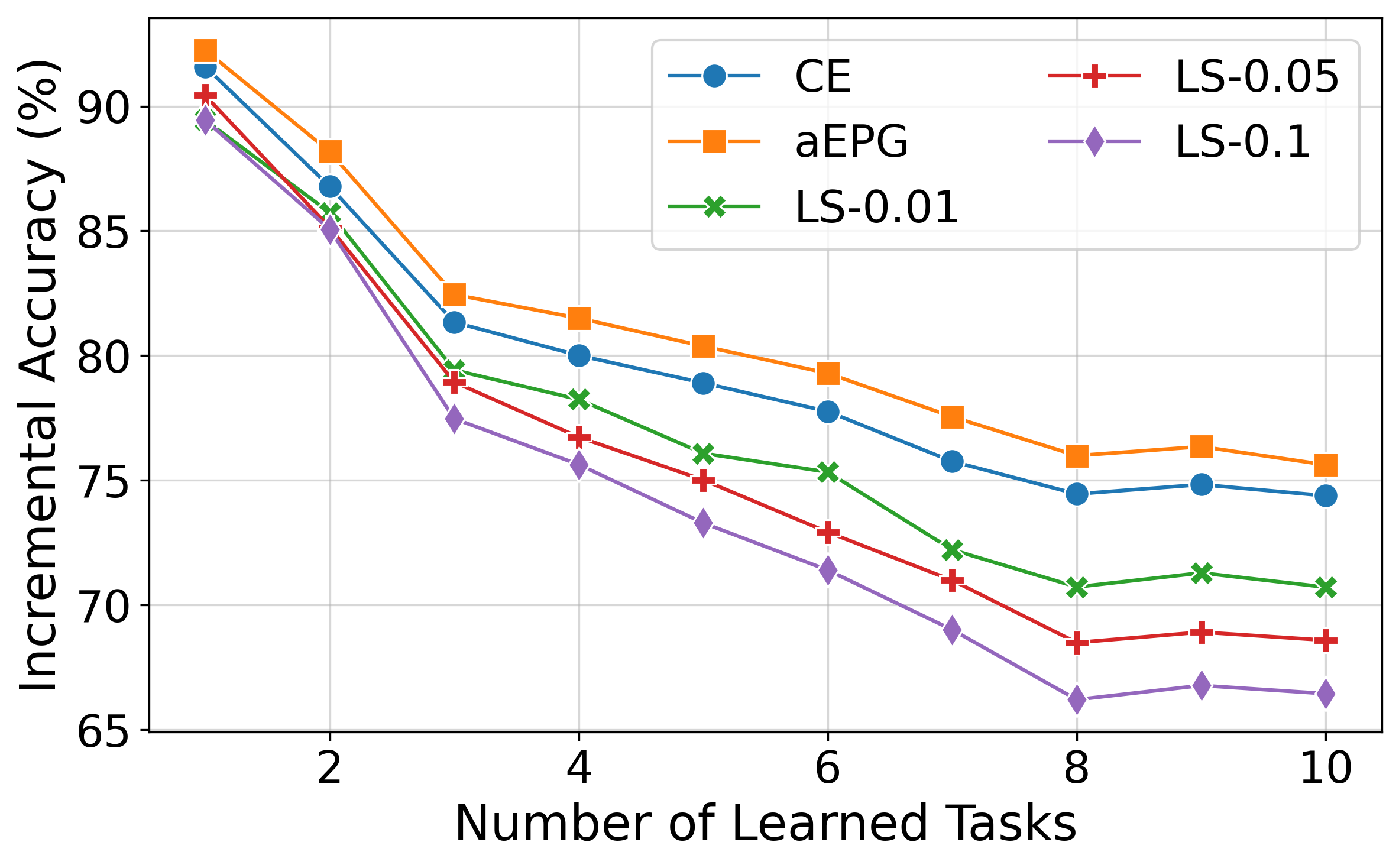}
\includegraphics[width=0.55\linewidth]{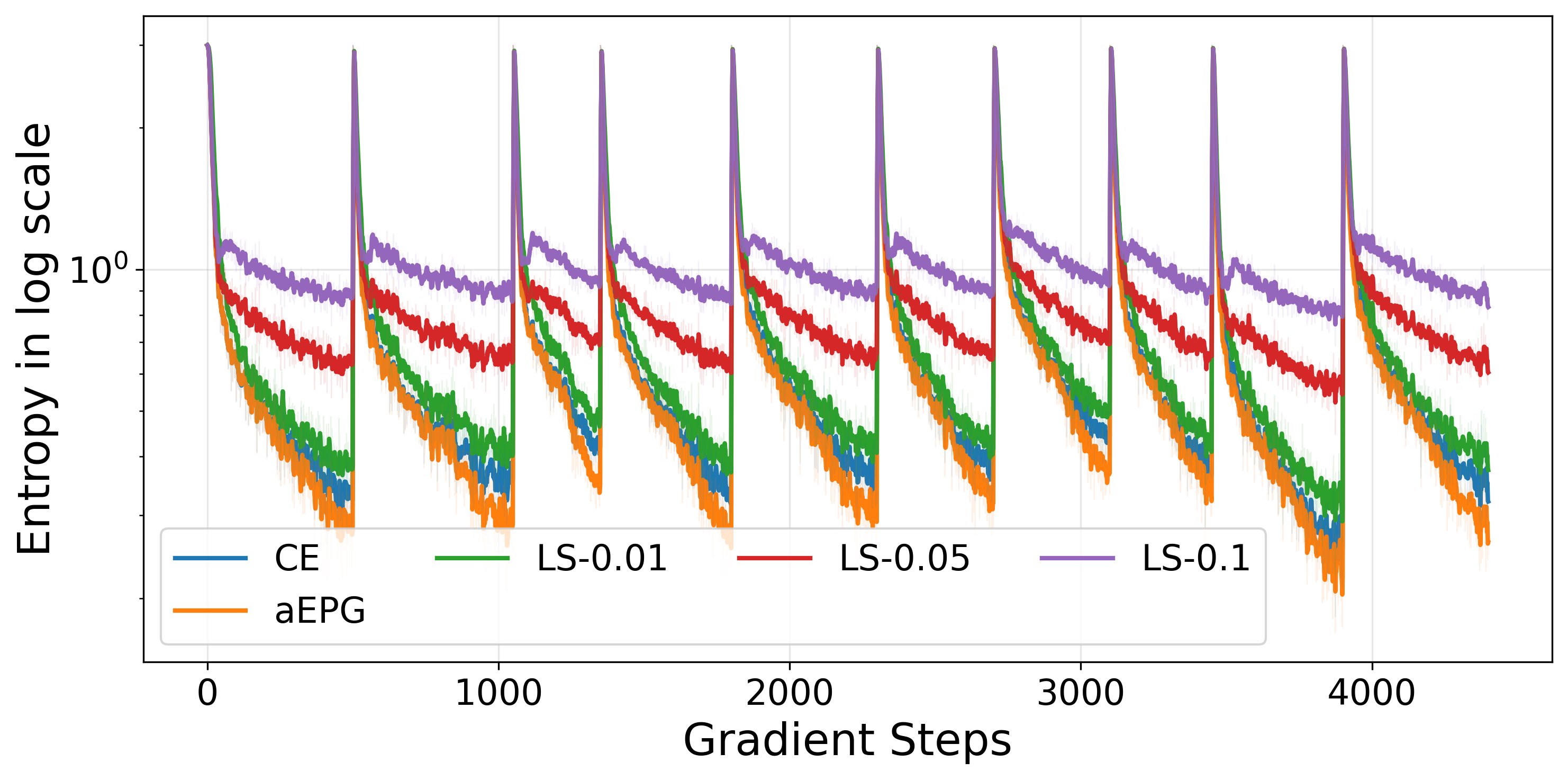}}
\subfigure[Confidence penalty]{\includegraphics[width=0.44\linewidth]{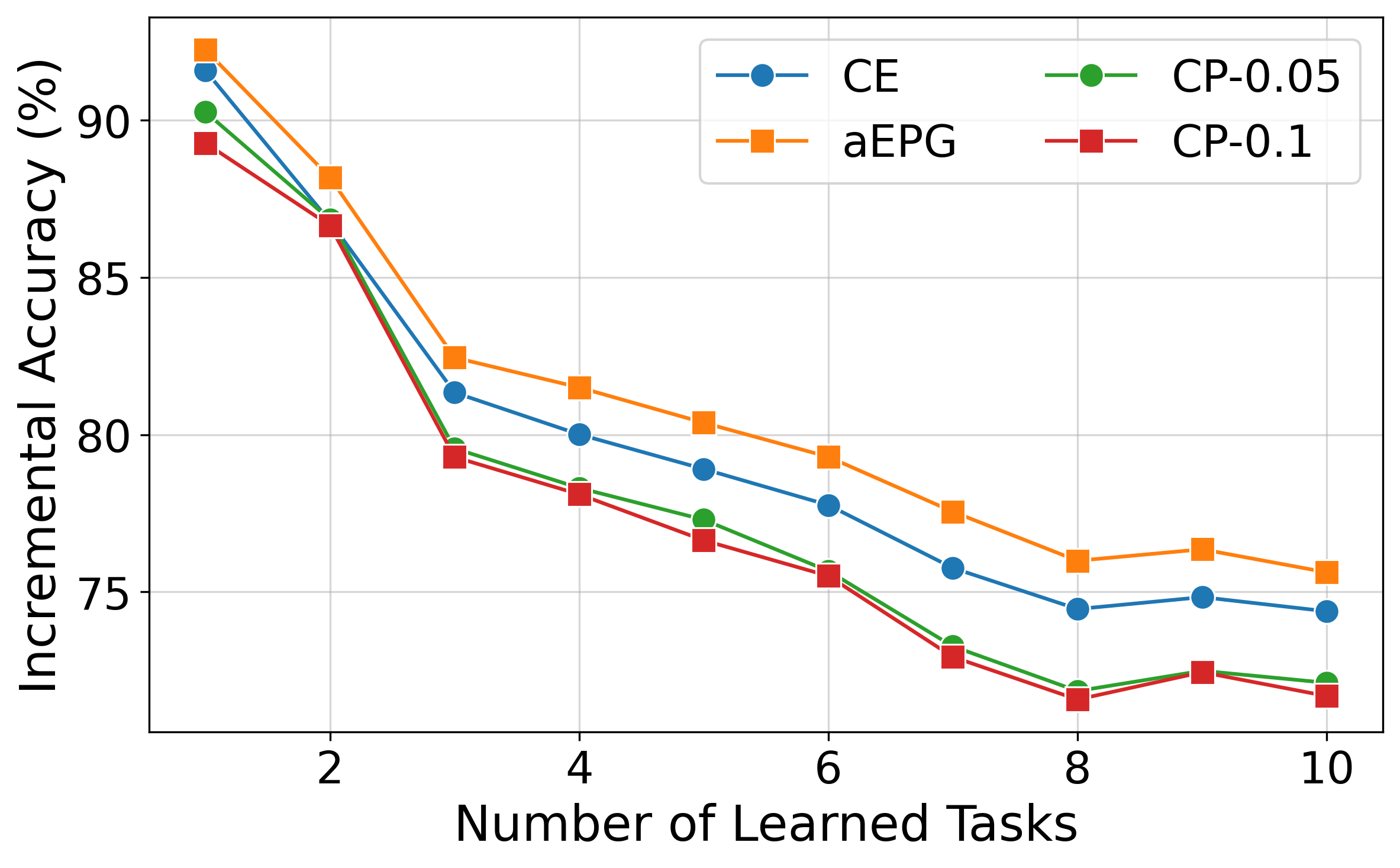}
\includegraphics[width=0.55\linewidth]{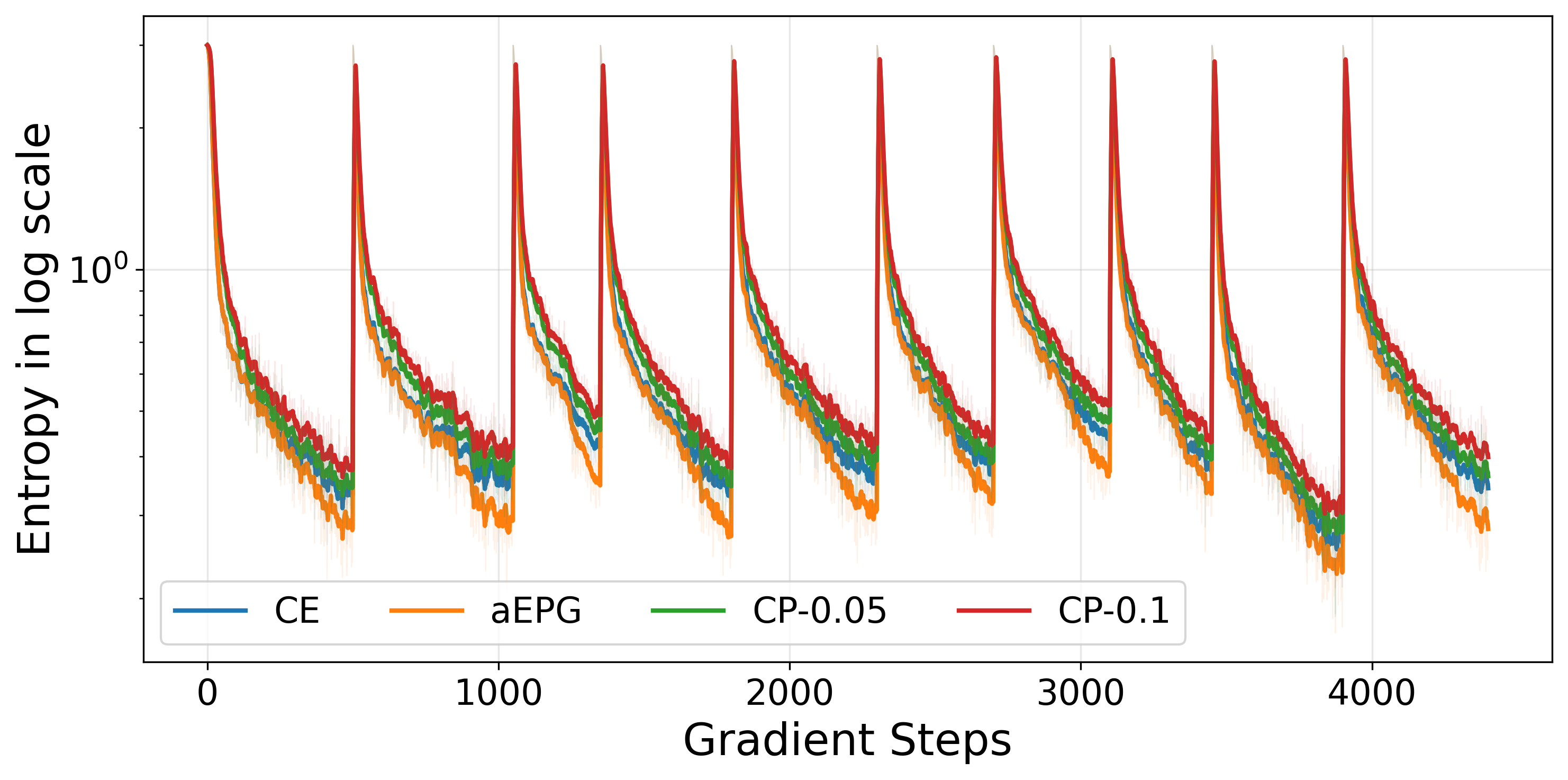}}
    \caption{The effect of entropy regularization strength in focal loss, label smoothing, confidence penalty with Split-ImageNetR.}
    \label{fig:focal_entropy_strength}
\end{figure}
\subsection{The effect of $\tau$}
    \begin{table}[h]
    \caption{The effect of $\tau$ in adaptive entropy annealing of aEPG}
    \vspace{4pt}
\centering
\begin{tabular}{lcccc}
    \toprule
         & $\alpha_{start}$ & $\alpha_{end}$ & Split-ImageNet-R & CLRS25 \\ \midrule
        $\tau = 4$ & 0.9820 & 0.0180 & 75.9 & 76.3 \\
        $\tau = 6$ & 0.9975 & 0.0025 & 75.9 & 76.2 \\
        $\tau = 8$ & 0.9997 & 0.0003 & 75.6 & 76.5 \\ \bottomrule
    \end{tabular}
\label{tab:tau}
\end{table}
\end{document}